\documentclass[11pt]{article}
\usepackage[a4paper,top=3cm,bottom=2cm,left=2cm,right=2cm,marginparwidth=1.75cm]{geometry}
\usepackage{fullpage}
\usepackage{graphics} 
\usepackage{epsfig}
\usepackage{amsmath} 
\usepackage{amssymb} 
\usepackage{amsthm}
\usepackage{tabularx,booktabs}
\usepackage{mathrsfs}
\usepackage{selectp}
\usepackage{Shorthands}
\usepackage[normalem]{ulem}
\usepackage{authblk}
\usepackage{subcaption}
\usepackage{natbib}
\usepackage{Shorthands}
\usepackage{algorithm}
\usepackage{algpseudocode}
\usepackage{wrapfig}
\usepackage{scalerel,stackengine}
\stackMath
\newcommand\reallywidehat[1]{%
\savestack{\tmpbox}{\stretchto{%
  \scaleto{%
    \scalerel*[\widthof{\ensuremath{#1}}]{\kern-.6pt\bigwedge\kern-.6pt}%
    {\rule[-\textheight/2]{1ex}{\textheight}}
  }{\textheight}%
}{0.5ex}}%
\stackon[1pt]{#1}{\tmpbox}%
}

\usepackage{xr}
\usepackage{xcolor}

\newboolean{showcomments}
\setboolean{showcomments}{true}
\ifthenelse{\boolean{showcomments}}
{ \newcommand{\mynote}[3]{
		\fbox{\bfseries\sffamily\scriptsize#1}
		{\small$\blacktriangleright$\textsf{\emph{\color{#3}{#2}}}$\blacktriangleleft$}}
	\newcommand{\zzz}[1]{{\setlength{\fboxsep}{2pt}\fcolorbox{black}{yellow}{\textsf{\emph{#1}}}}\xspace}}
{ \newcommand{\mynote}[3]{}
	\newcommand{\zzz}[1]{}}

\newcommand{\Jt}{J_{\theta^{(t)}}}

\usepackage{hyperref}
\hypersetup{       
    unicode=false,          
    pdftoolbar=true,       
    pdfmenubar=true,      
    pdffitwindow=false,     
    pdfstartview={FitH},   
    pdfnewwindow=true,      
    colorlinks=true,      
    linkcolor=blue,          
    citecolor=red,       
    filecolor=magenta,      
    urlcolor=blue,           
    breaklinks=true
}

\usepackage[toc,page,header]{appendix}
\usepackage{minitoc}

\title{ \bf On the Gradient Heterogeneity Dynamics of Adversarially Robust Federated Regression}
\date{}

\author[1]{Leonardo F. Toso\footnote{Correspondence to: \texttt{leonardo.toso@columbia.edu}.}}
\author[1]{James Anderson}
\author[2]{Nirupam Gupta}
\author[3]{Rafael Pinot}
\affil[1]{Columbia University, USA}
\affil[2]{University of Copenhagen, Denmark}
\affil[3]{Sorbonne Université and Université Paris Cité, CNRS, LPSM, France}

\begin{document}

\doparttoc 
\faketableofcontents

\maketitle

\begin{abstract}
Federated learning (FL) is intrinsically heterogeneous: honest clients may have different data-generating models. On top of that, adversarial clients can make heterogeneity even more pronounced by sharing arbitrary updates. Existing analyses typically control the interaction between statistical heterogeneity and adversarial behavior through gradient-dissimilarity conditions. However, the underlying bound is imposed a priori and may yield conservative guarantees even for least-squares regression. We instead derive the gradient heterogeneity from the statistical model of linear and nonlinear regression with fresh data samples at every round. Our bounds separate heterogeneity among the honest clients' ground-truth model parameters, finite-sample label noise, and initialization. We then demonstrate that, for any $(f,\kappa)$-robust aggregator with coefficient $\kappa = \mathcal{O}(f/n)$, where $f$ is the number of adversarial clients and $n$ the total number of clients (with $f/n < 1/2$),  convergence holds \emph{after an explicit sample burn-in}.
\end{abstract}

\allowdisplaybreaks

\section{Introduction}

Federated learning (FL) allows for a collection of clients to train a shared model while keeping their data locally \citep{mcmahan2017communication,kairouz2021advances}. This distributed architecture is particularly attractive when data are sensitive, geographically dispersed, or expensive to centralize. At the same time, it creates a fundamental robustness difficulty: a subset of clients may deviate arbitrarily from the prescribed protocol due to corrupted data, software or hardware faults, and even malicious behavior. Byzantine-robust FL aims to protect the learning process from such adversarial clients by replacing the aggregation average with $(f,\kappa)$-robust aggregators \citep{chen2017distributed,blanchard2017machine,yin2018byzantine,mhamdi2018hidden,guerraoui2024byzantine}, where $f$ is the number of adversarial clients and $\kappa >0$ is the robust aggregator coefficient. Therefore, any useful guarantee in the adversarial learning setting must account for optimization, statistical heterogeneity, and the underlying adversarial behavior.

Statistical heterogeneity makes the adversarial learning problem even more subtle. Honest clients generally have different data-generating models and may therefore share gradients that are genuinely far apart. A server must then distinguish this honest heterogeneity from an adversarial behavior without knowing which clients are honest. Prior analyses characterize this difficulty through a gradient-dissimilarity condition. The $G$-dissimilarity uniformly bounds the heterogeneity of honest gradients~\citep{karimireddy2020byzantine,wang2022fedadmm, allouah2023fixing, fallah2025adversarially}, while the $(G,B)$-dissimilarity (formally defined in \eqref{eq:GB}) allows the heterogeneity to grow with the norm of the honest-average gradient~\citep{li2020federated,karimireddy2020scaffold, allouah2023tight, gorbunov2023variance}. These have been shown to be useful heterogeneity abstractions, however, the bounds on $G$ and $B$ are imposed a priori, without any interpretability of such in terms of the underlying client model parameters, the client's data sample size, or the label noise level, resulting in conservative guarantees.

More precisely, \cite{allouah2023tight} demonstrate that $(G,B)$-dissimilarity provides a tight breakdown point at $1/(2+B^2)$ and a corresponding lower bound on the optimization error. Their upper-bound analysis requires $\kappa B^2<1$, for an arbitrarily large $B$, which is sharp for the class of objectives satisfying the $(G,B)$-dissimilarity. However, it leaves open how $G$ and $B$ arise from the statistical learning problem and whether such restriction on $\kappa = \mathcal{O}(f/n)$, with $n$ being the total number of clients, is necessary when gradient heterogeneity is characterized from first principles on the explicit statistical model (here we work with regression).

We address this question by studying the dynamics of gradient heterogeneity over the iterations of adversarially robust federated regression. For linear regression, for example, the gradient admits an exact decomposition into the client's model parameter heterogeneity, the finite-sample covariance concentration, and the label-noise level. The covariance concentration is the term that couples the current optimization error to the gradient heterogeneity, with coefficient decreasing as $1/\tau$, where $\tau$ is the number of fresh samples per client and round. Therefore, the analogue of $B^2$ in our analysis is not an arbitrary large constant set a priori: it is a concentration parameter that can be controlled through the batch size $\tau$. \\

\noindent\textbf{Informal result.} Let $\cH$ be the set of honest clients, let $T$ be the number of iterations/rounds, and assume that $\tau \geq \tau_{\textsf{burn-in}}$ (with $\tau_{\textsf{burn-in}}$ specified in Sections \ref{sec:linear_regression} and \ref{sec:nonlinear_regression}). 
Let $L_\mathcal{H}(\theta^{(t)}):= \frac{1}{|\mathcal{H}|}\sum_{i \in \mathcal{H}}L_i(\theta^{(t)})$, with  $L_i(\theta)$ being the regression loss evaluated at the model parameter $\theta$. Then, for a robust coefficient $\kappa = \mathcal{O}(f/n)$, with $f/n < 1/2$, the convergence bound holds
\begin{align*}
\frac{1}{T}\sum_{t=0}^{T-1} \|\nabla L_{\mathcal{H}}(\theta^{(t)})\|^2_2
&\lesssim\underbrace{\frac{\texttt{initialization error}}{T}}
_{\text{optimization}}+\left(\kappa+\frac{\omega^2}{|\cH|}\right)\times\underbrace{\texttt{model heterogeneity}}_{\text{bias}}\\
&+\left(\kappa+\frac{1}{|\cH|}\right)\times
\underbrace{\texttt{statistical error}}_{\text{finite samples}},
\end{align*}
where $\omega^2 =\mathcal{O}(1/\tau)$ is the covariate concentration parameter. The statistical error decreases as $1/\tau$, up to the problem dimension, logarithmic, and conditioning factors (omitted in $\lesssim$) stated in Sections \ref{sec:linear_regression} and \ref{sec:nonlinear_regression}. Hence, any fixed $\kappa$ is admissible once the fresh-batch size is sufficiently large (i.e., larger than $\tau_{\textsf{burn-in}}$).

\subsection{Contributions}

Our contributions are summarized as follows:\\

$\bullet$ We provide a first-principles characterization of gradient heterogeneity in adversarially robust federated regression, demonstrating that it decomposes into (i) honest client's model parameter heterogeneity, (ii) finite-sample statistical error, and (iii) an initialization-dependent error (Lemmas \ref{lem:bound_GT_linear} and \ref{lem:bound_GT_nonlinear_main}).\\

$\bullet$ Our main contribution is that, for any fixed $\kappa = \mathcal{O}(f/n)$, the non-asymptotic guarantees hold once each client uses a \emph{sufficiently large fresh data batch} (Theorems \ref{theorem:ergodic_convergence_linear} and \ref{thm:convergence_nl_mainb}). In contrast, prior work \cite{allouah2023tight} requires $\kappa B^2<1$, with $B>0$ being potentially prohibitively large. The price of removing such condition on $\kappa$ is explicit: it requires more samples per client.\footnote{The condition $\kappa B^2<1$ in \cite{allouah2023tight} is imposed in addition to the assumption that the fraction of adversarial clients satisfies $f/n < 1/2$.}\\

$\bullet$ We establish non-asymptotic parameter recovery error bounds that explicitly capture the interplay among honest client's model parameter heterogeneity, client's sample size, initialization, and robust aggregation (Corollaries \ref{cor:linear_recovery} and \ref{corollary_nonlinear}).\\

Our contributions connect the underlying statistical model of regression to the gradient heterogeneity abstraction in adversarially robust FL. Before formalizing our regression model, we collect the notation used throughout the paper.

\subsection{Notation} We use $\norm{\cdot}$ to denote the Euclidean norm for vectors and the spectral/operator norm for matrices, and $\norm{\cdot}_F$ to denote the Frobenius norm. $\mathbb{S}_F^{p\times q}:=\{U\in\mathbb{R}^{p\times q}:\|U\|_F=1\}$ denotes the unit sphere in $\mathbb{R}^{p\times q}$ equipped with the Frobenius norm. For a real-valued random variable $X$, its sub-Gaussian norm \cite[Definition 2.6.4]{vershynin2018high} is defined as follows:
\begin{align*}
\|X\|_{\psi_2} := \inf\left\{t>0 :
\mathbb{E}\left[ \exp\left(\frac{X^2}{t^2}\right) \right] \leq 2\right\}.
\end{align*}
The random variable $X$ is said to be sub-Gaussian if $\|X\|_{\psi_2}<\infty$. For a random vector $Z\in\mathbb{R}^d$ and a random matrix $M\in\mathbb{R}^{p\times q}$, we define
\begin{align*}
\|Z\|_{\psi_2}
&:=\sup_{u\in\mathbb{S}^{d-1}}\|\langle u,Z\rangle\|_{\psi_2},\\
\|M\|_{\psi_2}
&:=\sup_{U\in\mathbb{S}_F^{p\times q}}\|\langle U,M\rangle_F\|_{\psi_2}.
\end{align*}
$Z$ and $M$ are sub-Gaussian when its corresponding $\psi_2$-norm is finite.

We use the asymptotic notation, $h=\mathcal{O}(g)$ and $h=\Omega(g)$ denote upper and lower bounds up to positive constants for sufficiently large arguments, respectively. We write $h\lesssim g$ if $h\leq Cg$ for a constant $C>0$.

Next we introduce our statistical model and the adversarially robust FL setting with the $(f,\kappa)$ robust aggregation definition leveraged to control adversarial clients updates.

\section{Problem Formulation}\label{sec:problem_formulation}

We study a synchronous server-client FL architecture in which honest clients compute gradients from fresh local data batches and the server robustly aggregates them. The objective is the population honest-average loss $L_{\mathcal{H}}(\cdot)$, whose identities are unknown to the server. \\

\noindent \textbf{Clients and adversaries.}
There are $n$ clients indexed by $[n] = \{1,\dots,n\}$. An unknown subset $\cB \subset [n]$ of size $|\cB| = f$ is adversarial, with $f < \tfrac{n}{2}$. The remaining clients, indexed by the set $\cH = [n]\setminus\cB$, are honest.\\
 
\noindent \textbf{Data and samples.} At every iteration $t$, each honest client $i \in \cH$ draws a fresh batch of $\tau$ i.i.d. samples $\{(X_{i,k}^{(t)},Y_{i,k}^{(t)})\}_{k=1}^{\tau}$ from a client-specific distribution $\mathcal{D}_i$ on $\R^{\vd} \times \R^{\vdy}$. The batches are independent across clients and iterations. The fresh-sample assumption ensures that, conditionally on the current iterate $\theta^{(t)}$, the samples used to compute the next update remain independent. For simplicity, we suppress the iteration index on the samples whenever we analyze a fixed iteration.\\

\noindent \textbf{Losses.} For a parameter $\theta$ ($\theta \in \mathbb{R}^{q\times d}$ for linear regression, and $\theta \in \mathbb{R}^{p}$ for nonlinear regression), we define the population loss of client $i$ and the honest-average population loss as follows:
\begin{align*}
L_i(\theta)&:=\mathbb{E}_{(X,Y)\sim\mathcal{D}_i}
\bigl[\ell(\theta;(X,Y))\bigr],\\
L_\cH(\theta)&:=\frac{1}{|\cH|}\sum_{i\in\cH}L_i(\theta).
\end{align*}
At iteration $t$, client $i$ computes the fresh-batch empirical loss
\begin{align*}
\hat L_i^{(t)}(\theta) &:= \frac{1}{\tau}\sum_{k=1}^{\tau}
\ell\bigl(\theta;(X_{i,k}^{(t)},Y_{i,k}^{(t)})\bigr),\\
\hat L_\cH^{(t)}(\theta) &:= \frac{1}{|\cH|}
\sum_{i\in\cH}\hat L_i^{(t)}(\theta),
\end{align*}
where throughout this paper we consider the squared losses
\begin{align*}
\textbf{Linear:}\quad
&\ell(\theta;(X,Y)):=\norm{Y-\theta X}_2^2,\\
\textbf{Nonlinear:}\quad
&\ell(\theta;(X,Y)):=\norm{Y-h_\theta(X)}_2^2,
\end{align*}
where $h_\theta(X)$ is a nonlinear function of the covariates $X \in \mathbb{R}^\vd$ parameterized by $\theta \in \mathbb{R}^{ p}$.
 
\begin{assumption}\label{assumption:isotropic}
$\Sigma_i :=\mathbb{E}_{X_i\sim\mathcal{D}_i}[X_iX_i^\top] = I_{\vd}$, for all honest clients $i \in \cH$.   
\end{assumption} 

The isotropic data assumption is standard in non-asymptotic analyses of linear and nonlinear regression and can be obtained by whitening when the population covariance is known and nonsingular \citep{vershynin2018high,oymak2019overparameterized}. It allows us to isolate heterogeneity in the client ground-truth model parameter rather than in the conditioning of their covariates. For nonlinear-regression, we additionally assume a common covariate marginal across honest clients.

Next we describe how the server combines the client gradients in the presence of arbitrary adversarial updates.

\subsection{Adversarially Robust FL}

We consider the class of $(f,\kappa)$-robust aggregators defined below.

\begin{definition}[($f,\kappa$)-robust aggregator]
\label{def:robust-agg}
An aggregation rule $\mathsf{F}:\left(\mathbb{R}^p\right)^n\to\mathbb{R}^p$ is called $(f,\kappa)$-robust, if for every set of inputs $a_1,\dots,a_n \in \mathbb{R}^{p}$ and every honest set $\mathcal{H}$ with $|\mathcal{H}|=n-f$,
\begin{align*}
\left\|\mathsf{F}(a_1,\dots,a_n)-\bar a_{\mathcal{H}}\right\|^2
\leq  \frac{\kappa}{|\mathcal{H}|}\sum_{i\in\mathcal{H}}\|a_i-\bar a_{\mathcal{H}}\|^2.    
\end{align*}
\end{definition}
The parameter $\kappa$ measures how strongly the aggregation error reacts to the deviation of the honest inputs (i.e., gradient updates). The definition is deterministic and does not restrict how adversarial clients choose their updates. Therefore, once the gradient heterogeneity is controlled, the same optimization argument applies to any aggregation rule satisfying Definition \ref{def:robust-agg}, that is, the statistical model enters only through our bound on the honest-gradient heterogeneity. 

Well-known instances of $(f,\kappa)$-robust aggregators include Krum~\citep{blanchard2017machine}, coordinate-wise trimmed mean (CWTM) \citep{yin2018byzantine}, geometric median (GM), minimum diameter averaging, etc.. In addition, combined with the nearest neighbor mixing (NNM) approach of \citep{allouah2023fixing}, any $(f,\kappa)$-robust aggregator achieves $\kappa = \mathcal{O}(f/n)$, which is information-theoretically optimal as proved in \cite{allouah2023fixing}.

The adversarially robust gradient descent update of $\theta^{(t)}$, for all iterations $t \in \{0,1,\ldots,T-1\}$, is given by
\begin{align}\label{eq:update}
\theta^{(t+1)} = \theta^{(t)}-\eta\sF\bigl(
&\nabla\hat L_1^{(t)}(\theta^{(t)}),\ldots,\nabla\hat L_n^{(t)}(\theta^{(t)})\bigr).
\end{align}
For simplicity, we use $\sF^{(t)}$ to denote the aggregated fresh-batch gradients at iteration $t$ in \eqref{eq:update}.

\subsection{The Role of Gradient Heterogeneity}

An important quantity in the convergence analysis of adversarially robust FL is the gradient heterogeneity and its average over iterations given by
\begin{align}
G^{(t)} &:= \frac{1}{|\cH|}\sum_{i\in\cH}
\norm{\nabla_\theta \hat L^{(t)}_i(\theta^{(t)}) - \nabla_\theta \hat L^{(t)}_\cH(\theta^{(t)})}^2.
\notag \\
G_T&:=\frac{1}{T}\sum_{t=0}^{T-1}G^{(t)}.
\end{align}

When controlling $$Q_T:= \frac{1}{T}\sum_{t=0}^{T-1}\norm{\nabla L_\cH(\theta^{(t)})}^2$$ through our ergodic convergence analysis, $G_T$ will appear multiplied by the robustness coefficient $\kappa$ of the aggregation rule, i.e., after leveraging the definition of $(f,\kappa)$-robust aggregators. Therefore, understanding the structure of $G_T$ is of utmost importance for our analysis.

As discussed previously, prior analyses set the bound for $G_T$ with a priori fixed constants $G$ and $B$. For instance, $G$-dissimilarity \citep{blanchard2017machine,yin2018byzantine}, assumes there exists a uniform $G^2 < \infty$ such that $G^{(t)} \leq G^2$ for all iterations $t = 0,1,\ldots,T-1$ and all $\theta^{(t)}$. As \cite{allouah2023tight} points out, $G$-dissimilarity fails for least-squares regression whenever clients have different data-generating distributions (even when $f=0$), as the gradient heterogeneity grows when $\theta^{(t)}$ moves away from the optima of different client's model parameters $\theta^\star_i$, for $i \in \mathcal{H}$. 

Therefore, \cite{allouah2023tight}, following \cite{li2020federated,karimireddy2020scaffold}, leverage $(G,B)$-gradient dissimilarity:
\begin{align}\label{eq:GB}
G_T \leq G^2 + \frac{B^2}{T}\sum_{t=0}^{T-1}\norm{\nabla L_\cH(\theta^{(t)})}^2,
\end{align}
which allows the gradient heterogeneity to grow with the honest averaged gradient norm at rate $B^2$. Under this dissimilarity condition, \cite{allouah2023tight} prove a tight lower bound $\Omega\left(\tfrac{f/n}{(1-(2+B^2)f/n)} G^2\right)$ on the optimization error of any robust algorithm, and demonstrate that the largest fraction of adversarial clients should be $\tfrac{1}{2+B^2}$ rather than $\tfrac{1}{2}$.
 
Here, we will shed light on the interpretability of $G$ and $B$. We explicitly characterize the structure of $G_T$ for linear and nonlinear regression under isotropic data. In particular, for linear regression, we  will demonstrate that
\begin{align*}
G^2 &\lesssim \frac{1}{|\cH|^2}\sum_{i\in\cH}\sum_{j\in\cH}
\norm{\thetastar_i - \thetastar_j}^2 + \texttt{statistical error}
\end{align*}
and $B^2 = \mathcal{O}\left(1/\tau\right)$, therefore, demonstrating that with a sufficiently large amount of data samples per client, the condition on $\kappa B^2 < 1$ can be guaranteed by the number of data samples per client $\tau$ and not by the robust aggregator coefficient $\kappa$. Such interpretability of $G$ and $B$ is the result of our first-principles analysis of the gradient dynamics in adversarially robust federated regression.

We first focus on the linear regression problem, where the honest gradients admit an exact decomposition that makes the source of each error term transparent in the analysis.

\section{Linear Regression}\label{sec:linear_regression}

The linear setting provides the cleanest illustration of our approach. We first control the empirical gradient heterogeneity, then use this control in the descent recursion, and finally translate approximate stationarity into parameter recovery error.

Each honest client $i\in\cH$ has an unknown ground-truth parameter $\thetastar_i \in \R^{\vdy\times\vd}$ and is assumed to generate $\tau$ samples according to
\begin{align}\label{eq:linear_model}
Y_{i,k} = \thetastar_i X_{i,k} + V_{i,k} \text{ for all } k=1,\dots,\tau,
\end{align}
where $X_{i,k}\in\R^{\vd}$ satisfies $\norm{X_{i,k}} \leq R$ almost surely. 
Here $V_{i,k}\in\R^{\vdy}$ denotes the vector-valued label noise.

\begin{assumption}
\label{assm:subg}
For each honest client $i$, sample $k$, and iteration $t$, the noise $V_{i,k}^{(t)}\in\R^{\vdy}$ is conditionally mean-zero and $\sigma^2$-sub-Gaussian given $X_{i,k}^{(t)}$, i.e., for every $u\in\mathbb{S}^{\vdy-1}$,
\begin{align*}
\mathbb{E}\left[\exp \left(\lambda\,u^\top V_{i,k}^{(t)}\right)\mid X_{i,k}^{(t)}\right] \leq \exp \left(\tfrac{\lambda^2\sigma^2}{2}\right), \; \forall\lambda\in\R.
\end{align*}
\end{assumption}

We emphasize that conditional mean-zero sub-Gaussian noise is standard in high-dimensional regression and yields dimension-explicit concentration for the empirical gradient \citep{vershynin2018high}.

At iteration $t$, the fresh-batch empirical squared loss for client $i \in \mathcal{H}$ is given by
\begin{align*}
\hat L_i^{(t)}(\theta) := \frac{1}{\tau}\sum_{k=1}^{\tau}\norm{Y_{i,k}^{(t)} - \theta X_{i,k}^{(t)}}^2,
\end{align*}
and our goal is to minimize the honest-average population loss
\begin{align*}
\min_{\theta \in \R^{\vdy \times \vd}} L_\cH(\theta) := \frac{1}{|\cH|}\sum_{i\in\cH}\mathbb{E}\left[\norm{Y_i-\theta X_i}^2\right].
\end{align*}

Let $\Gamma_{\textsf{lin}}:=\frac{1}{|\cH|^2}\sum_{i\in\cH}\sum_{j\in\cH}
\norm{\thetastar_i-\thetastar_j}^2 < \infty,$ denote the model heterogeneity.\\ 

\noindent \textbf{Gradient Decomposition.} We fix any iteration $t \in \{0,1,\ldots,T-1\}$ and model parameter $\theta^{(t)}$, and write
\begin{align}
\nabla_\theta \hat  L_i^{(t)}(\theta^{(t)}) - \nabla_\theta \hat  L_\cH^{(t)}(\theta^{(t)})
&= \frac{2}{|\cH|}\sum_{j\in\cH}(\thetastar_j - \thetastar_i)\textcolor{black}{\Sigma_i} \notag + \frac{2}{|\cH|}\sum_{j\in\cH}(\thetastar_j - \thetastar_i)(\hSigma_i - \Sigma_i) - (\xi_i - \barxi)\\
& + \frac{2}{|\cH|}\sum_{j\in\cH}\left[(\theta^{(t)}-\thetastar_j)(\hSigma_i - \Sigma_i) + (\theta^{(t)}-\thetastar_j)(\Sigma_j - \hSigma_j) \right]\notag,
\end{align}
where $\widehat \Sigma_i = \frac{1}{\tau}\sum_{k=1}^\tau X_{i,k}X^\top_{i,k}$, $\forall i \in \mathcal{H}$. The noise term is composed of $\xi_i := \frac{2}{\tau}\sum_{k=1}^{\tau}V_{i,k}X_{i,k}^\top \in \R^{\vdy\times\vd}$ with the honest-average noise term given by $\bar{\xi} = \frac{1}{|\mathcal{H}|}\sum_{i \in \mathcal{H}}\xi_i.$

In addition, let $\Delta L^{(0)}:=L_\cH(\theta^{(0)})-L_\cH(\thetastar)$, where $\thetastar$ minimizes the honest-average population loss $L_\cH(\theta)$.

\begin{lemma}[Gradient Heterogeneity]\label{lem:bound_GT_linear}
Suppose that Assumptions \ref{assumption:isotropic} and \ref{assm:subg} hold. Let the step-size be such that $\eta \leq \frac{1}{2}$. Given constants $C_0,C_1 >0$, for every $\delta\in(0,1)$, suppose that the number of fresh samples per client satisfies $\tau \geq \tau_{\textsf{burn-in}}$ with 
\begin{align*}
\frac{\tau_{\textsf{burn-in}}}{{C_0}C_1} \hspace{-0.05cm}:= \hspace{-0.05cm} \textcolor{black}(R^2\hspace{-0.05cm}+\hspace{-0.05cm}1)^2
\max\left\{\hspace{-0.05cm}1,\kappa\hspace{-0.05cm}+\hspace{-0.05cm}\frac{1}{|\cH|}\hspace{-0.05cm}\right\}
\hspace{-0.05cm}\log\hspace{-0.1cm}\left(\hspace{-0.05cm}\frac{2|\cH|T\textcolor{black}{(\vd\hspace{-0.05cm}+\hspace{-0.05cm}q)}}{\delta}\hspace{-0.05cm}\right).
\end{align*}
For some $\omega>0$, let
\begin{align*}
\omega^2:=C_1(R^2+1)^2
\left(\frac{\log(2|\cH|T\textcolor{black}{(\vd+q)}/\delta)}{\tau}\right).
\end{align*}
Then, with probability at least $1-\delta$, it holds that
\begin{align}\label{gradient_heterogeneity_mb}
G_T\lesssim\Gamma_{\textsf{lin}}
+\omega^2Q_T+R^2\sigma^2\left(\frac{\vd\vdy+\log(2|\cH|T/\delta)}{\tau}\right),
\end{align}
and consequently, it holds
\begin{align*}
G_T&\lesssim \Gamma_{\textsf{lin}} +\omega^2\frac{\Delta L^{(0)}}{T}+R^2\sigma^2
\left(\frac{\vd\vdy+\log(2|\cH|T/\delta)}{\tau}\right).
\end{align*}
\end{lemma}

\noindent \textbf{Proof sketch.} We apply the matrix Bernstein inequality \citep[Theorem~1.4]{tropp2012user} to the mean-zero matrices $X_{i,k}X_{i,k}^\top-I_{\vd}$, which yields $\|\widehat\Sigma_i-\Sigma_i\|\leq\omega/2$ simultaneously over the honest clients and iterations. For the noise matrices $V_{i,k}X_{i,k}^\top$, we apply the sub-Gaussian Hoeffding inequality \citep[Theorem~2.7.3]{vershynin2018high} to their scalar projections and take a union bound over a net of the Frobenius unit sphere. By substituting these bounds into the gradient decomposition above, then squaring and averaging over $i\in\cH$, gives
\begin{align*}
G^{(t)}\lesssim \Gamma_{\textsf{lin}}
+\omega^2\Delta^{(t)}
+R^2\sigma^2\left(\frac{\vd\vdy+\log(2|\cH|T/\delta)}{\tau}\right),
\end{align*}
where $\Delta^{(t)}:=\frac{1}{|\mathcal{H}|}\sum_{j\in\cH}\|\theta^{(t)}-\thetastar_j\|^2$. Then, under isotropic data, the variance decomposition and $\nabla L_\cH(\theta)=2(\theta-\bar\theta^\star)$ imply
$\Delta^{(t)}=\frac14\|\nabla L_\cH(\theta^{(t)})\|^2+\frac12\Gamma_{\textsf{lin}}$. Averaging over $t$ yields \eqref{gradient_heterogeneity_mb}, and substituting the convergence bound below gives the second expression. The complete proof is provided in Appendix \ref{appendix:linear}.\\

\noindent\textbf{Discussion of Lemma \ref{lem:bound_GT_linear}.}
We emphasize that this lemma is the mechanism behind our main result. Note that the iterate-dependent term (i.e., second term) in \eqref{gradient_heterogeneity_mb} is multiplied by $\omega^2=\mathcal{O}(1/\tau)$, while model heterogeneity and label noise enter additively. Therefore, for any fixed $\kappa = \mathcal{O}(f/n)$, with $f/n < 1/2$, increasing the fresh-batch size makes the term later multiplied by $\kappa$ small enough to be absorbed in the convergence argument, namely, when we bound $Q_T$.

With the heterogeneity term controlled, we can insert it into the robust descent recursion and obtain the following ergodic convergence guarantee.

\begin{theorem}[Convergence Bound]\label{theorem:ergodic_convergence_linear}
Suppose that the conditions of Lemma \ref{lem:bound_GT_linear} hold with $\eta = \frac{1}{2}$. Then, for every $\delta\in(0,1)$, with probability at least $1-\delta$, it holds that
\begin{align*}
Q_T&\lesssim\frac{\Delta L^{(0)}}{T}+\left(\kappa+\frac{\omega^2}{|\cH|}\right)\Gamma_{\textsf{lin}}+R^2\sigma^2\left(\kappa+\frac{1}{|\cH|}\right)
\left(\frac{\vd\vdy+\log(2|\cH|T/\delta)}{\tau}\right).
\end{align*}
\end{theorem}

\begin{proof}
The proof is provided in Appendix \ref{appendix:linear}.
\end{proof}

\noindent\textbf{Discussion of Theorem \ref{theorem:ergodic_convergence_linear}.}
This theorem makes our tradeoff explicit: $\kappa$ is not required to lie below a fixed threshold, e.g., $\kappa B^2 < 1$ as in \cite{allouah2023tight}. For any fixed $\kappa$, the burn-in condition in Lemma \ref{lem:bound_GT_linear} is sufficient, and a larger $\kappa$ is compensated by a larger sample size $\tau$. After this burn-in, the remaining error consists of the model parameter heterogeneity bias, a statistical term that decreases with $\tau$, and an initialization term that goes away with $T.$\\

\noindent\textbf{Key Takeaway.} Our bounds separate the irreducible model parameter heterogeneity $\Gamma_{\textsf{lin}}$ from the finite-sample statistical error and initialization-dependent error. The burn-in condition \eqref{eq:tau_condition} demonstrates explicitly how the required fresh-batch size grows with the robust coefficient $\kappa$.
More precisely, the averaged recovery error contains three different contributions. The term $\Delta L^{(0)}/ T$ is transient and characterizes the effect of initialization. The term proportional to $\Gamma_{\textsf{lin}}$ is irreducible when honest clients have different ground-truth parameters, even without adversaries (and it goes away when $\tau \to \infty$ as $\omega \to 0$). The last term is statistical and decreases as $1/\tau$, up to logarithmic factors. We also note that robust aggregation amplifies the two non-transient terms through $\kappa$.

The comparison with \cite{allouah2023tight} is the most relevant to this work. Their $(G,B)$-dissimilarity condition controls the iterate-dependent heterogeneity through an a priori fixed coefficient $B^2$ and therefore requires $\kappa B^2 <1$. Here, the corresponding coefficient $\omega^2$ scales as $\mathcal{O}(1/\tau)$, up to problem-dependent and logarithmic factors. We also refer the reader to \eqref{comparison_GB}, where we write our resulting gradient heterogeneity bound in the $(G,B)$ form:
\begin{align*}
G_T&\lesssim\underbrace{\Gamma_{\textsf{lin}}
+R^2\sigma^2\left(
\frac{\vd\vdy+\log(2|\cH|T/\delta)}{\tau}\right)}_{:=G^2}+\underbrace{\omega^2}_{:=B^2}\frac{1}{T}\sum_{t=0}^{T-1}
\norm{\nabla L_\cH(\theta^{(t)})}^2.
\end{align*}

Moreover, the burn-in condition ensures directly that $\kappa\omega^2$ is sufficiently small. Thus, every fixed $\kappa$ is admissible \emph{after a sufficiently large sample burn-in}. For aggregators with $\kappa$ scaling as $\mathcal{O}(f/n)$, the adversarial contribution has the optimal linear dependence on the adversarial fraction, while the usual requirement $f<n/2$ remains for the existence of such robust aggregator.

The bound on $Q_T$ also provides a direct parameter recovery error bound under isotropic data as given in the corollary below.

\begin{corollary}[Parameter Recovery Error]\label{cor:linear_recovery}
Suppose that the conditions of Theorem \ref{theorem:ergodic_convergence_linear} hold. Then, with probability at least $1-\delta$,
\begin{align*}
\frac{1}{T}\sum_{t=0}^{T-1}\frac{1}{|\cH|}\sum_{j\in\cH}
\norm{\theta^{(t)}-\thetastar_j}^2
&=\frac{1}{4}Q_T+\frac{1}{2}\Gamma_{\textsf{lin}}\\
&\lesssim\frac{\Delta L^{(0)}}{T}
+\left(1+\kappa+\frac{\omega^2}{|\cH|}\right)
\Gamma_{\textsf{lin}}\\
&+R^2\sigma^2\left(\kappa+\frac{1}{|\cH|}\right)\left(
\frac{\vd\vdy+\log(2|\cH|T/\delta)}{\tau}\right).
\end{align*}
\end{corollary}

\begin{proof}
We refer the reader to Appendix \ref{appendix:linear} for the proof.
\end{proof}

\section{Nonlinear Regression}\label{sec:nonlinear_regression}

We now extend the preceding argument beyond the linear parameterization. The proof follows the same high-level sequence, but model parameter heterogeneity is replaced by heterogeneity between the client nonlinear functions and covariance concentration $\|\Sigma_i - \widehat \Sigma_i\|$, for all $i \in \mathcal{H}$, is replaced by the concentration of the residual-Jacobian process.

Each honest client $i\in\cH$ has ground-truth parameter $\thetastar_i\in\R^p$ (i.e., a vector now) and generates data according to
\begin{align*}
Y_{i,k} = h_{\thetastar_i}(X_{i,k}) + V_{i,k}, \forall k=1,\dots,\tau,
\end{align*}
where $h_\theta:\R^{\vd}\to\R^{\vdy}$ is a nonlinear function parameterized by $\theta\in\R^p$, $V_{i,k}\in\R^{\vdy}$
satisfies Assumption \ref{assm:subg}, and $\|X_{i,k}\|\leq R$ almost surely. The empirical loss is then given by
\begin{align*}
L_i(\theta) := \frac{1}{\tau}\sum_{k=1}^{\tau}\norm{Y_{i,k} - h_\theta(X_{i,k})}^2.
\end{align*}

For the nonlinear analysis, we assume that the covariates of all honest clients have the same marginal distribution. The client heterogeneity therefore enters through the ground-truth functions $h_{\thetastar_i}(\cdot)$, while every expectation denoted by $\mathbb{E}_X$ is taken with respect to this common covariate distribution.
 
\begin{assumption}\label{assm:jacobian}
The Jacobian given by $J_\theta(X) := \nabla_\theta h_\theta(X)\in\R^{p\times\vdy}$ satisfies
$\norm{J_\theta(X)}\leq\bar{J}$ for all $\theta,X$.
\end{assumption}

Let ${E}_j^{(t)} := \mathbb{E}_X\bigl[\|h_{\theta^{(t)}}(X) - h_{\theta_j^\star}(X)\|^2\bigr]$, for $j \in \mathcal{H}.$

\begin{assumption}\label{assumption:subGaussian_residual}
For every honest client $j\in \mathcal H$ and every iteration $t$, conditionally on the current iterate $\theta^{(t)}$, the residual
$$
r_j^{(t)} := \|h_{\theta^{(t)}}(X)-h_{\theta_j^\star}(X)\|,
$$
is sub-Gaussian, such that
$$
\|r_j^{(t)}\|_{\psi_2}  \leq C_2
\left( \mathbb{E}_X\left[ \left(r_j^{(t)}\right)^2 \right]\right)^{1/2} =
C_2 \sqrt{ E_j^{(t)}},
$$
for some constant $C_2 > 0$.
\end{assumption}

We note that along with the bounded Jacobian assumption (Assumption \ref{assm:jacobian}) this implies that matrices 
$$ W_{j,k}^{(t)}-\mathbb{E} W_{j,k}^{(t)}
\text{ with } W_{j,k}^{(t)} := r_{j}^{(t)}(X_{j,k})J_{\theta^{(t)}}(X_{j,k})^\top$$ 
satisfy
\begin{align*}
\left\|  W_{j,k}^{(t)}-\mathbb{E} W_{j,k}^{(t)} \right\|_{\psi_2} \leq C_3\bar J\sqrt{E_j^{(t)}},
\end{align*} 
for some constant $C_3 >0$.

We emphasize that Assumptions \ref{assm:jacobian} and \ref{assumption:subGaussian_residual} are standard regularity conditions in nonlinear statistical learning and empirical stochastic process analyses. In particular, Assumption \ref{assumption:subGaussian_residual} corresponds to the sub-Gaussian condition on a function class, under which the $\psi_2$-norm of function differences is controlled by their $L_2$-norm~\citep{lecue2013learning,mendelson2015learning}. The bounded Jacobian assumption imposes a uniform Lipschitz condition on the model with respect to the parameter \citep{oymak2019overparameterized, schmidt2020nonparametric} and holds, for example, for bounded inputs and bounded or Lipschitz neural networks on compact parameter sets.

\begin{assumption}\label{ass:smooth}
The honest-average population loss $L_\cH$ has an $L^\prime$-Lipschitz gradient, i.e., for all $\theta,\theta'\in\R^p$,
\begin{align*}
\norm{\nabla L_\cH(\theta')-\nabla L_\cH(\theta)}
\leq L^\prime\norm{\theta'-\theta}.
\end{align*}
\end{assumption}

\begin{assumption}\label{ass:PL}
The honest-average loss $L_\cH(\theta)$ satisfies the PL condition with constant $\mu^\prime>0$, i.e.,
\begin{align}\label{eq:PL}
\norm{\nabla_\theta L_\cH(\theta)}^2
\geq  2\mu^\prime\bigl(L_\cH(\theta) - L_\cH(\thetastar)\bigr),\forall\,\theta\in\R^p.
\end{align}
\end{assumption}

Assumption \ref{ass:PL} is standard in nonlinear optimization and strictly weaker than strong convexity, i.e., every $\mu$-strongly convex function satisfies PL with $\mu^\prime=\mu$, but not vice versa. In the linear setting $h_\theta(X)=\theta X$, the loss $L_\cH(\theta)$ is $2$-strongly convex under isotropy, hence satisfies \eqref{eq:PL} with $\mu^\prime=2$.
In the nonlinear setting, the PL condition holds for sufficiently overparameterized models whenever the neural tangent kernel is positive definite \citep{jacot2018neural}.

Let
\begin{align*}
\Gamma_{\textsf{nonlin}}:=\frac{1}{|\cH|^2}\sum_{i\in\cH}\sum_{j\in\cH}
\mathbb{E}_X\left[\norm{h_{\thetastar_i}(X)-h_{\thetastar_j}(X)}^2\right] < \infty,
\end{align*}
denote the honest client's model heterogeneity in nonlinear function space.

With these regularity assumptions in place, we are now ready to bound the empirical gradient heterogeneity in terms of the function-space heterogeneity $\Gamma_{\textsf{nonlin}}$, the statistical error, and $Q_T$ (which can also be made in terms of a initialization-dependent term).

\begin{lemma}[Gradient Heterogeneity]\label{lem:bound_GT_nonlinear_main}
Suppose that Assumptions \ref{assm:subg}, \ref{assm:jacobian}, \ref{assumption:subGaussian_residual}, \ref{ass:smooth}, and \ref{ass:PL} hold. Let the step-size be selected as $\eta\leq1/L^\prime$. Given constants $C_0,C_1 >0$. For every $\delta\in(0,1)$, suppose that the number of fresh samples per client satisfies $\tau \geq \tau_{\textsf{burn-in}}$,
\begin{align*}
\frac{\tau_{\textsf{burn-in}}}{\textcolor{black}{C_0}C_1}:=\frac{\bar J^2}{\mu^\prime}
\left(\kappa+\frac{1}{|\cH|}\right)
\left(\textcolor{black}{p}+\log\left(\frac{2|\cH|T}{\delta}\right)\right).
\end{align*}
In addition, for some $\bar\omega>0$, let
\begin{align*}
\bar\omega^2:=C_1\bar J^2\left(\frac{\textcolor{black}{p}+\log(2|\cH|T/\delta)}{\tau}\right).
\end{align*}
Then, with probability at least $1-\delta$, it holds
\begin{align}\label{G_T_Q_T_nonlinear}
G_T\lesssim (\bar J^2+\bar\omega^2)\Gamma_{\textsf{nonlin}}&+\frac{\bar\omega^2 Q_T}{2\mu^\prime}+\bar J^2\sigma^2\left(\frac{\textcolor{black}{p}+\log(2|\cH|T/\delta)}{\tau}\right).
\end{align}
and consequently, it holds that
\begin{align*}
G_T\lesssim\textcolor{black}{(\bar J^2+\bar\omega^2)}
\Gamma_{\textsf{nonlin}} &+ \frac{\bar\omega^2\Delta L^{(0)}}{\mu^\prime T}+\bar J^2\sigma^2
\left(\frac{\textcolor{black}{p}+\log(2|\cH|T/\delta)}{\tau}\right).
\end{align*}

\end{lemma}

\begin{proof} We defer the proof to Appendix \ref{appendix:nonlinear}.
\end{proof}

\noindent \textbf{Discussion of Lemma \ref{lem:bound_GT_nonlinear_main}.} The nonlinear heterogeneity bound mirrors the linear decomposition in function space. The first term, $\bar J^2\Gamma_{\textsf{nonlin}}$, measures the heterogeneity among the honest ground-truth functions and persists even with infinite data. The label-noise term decreases with the fresh-batch size. The middle term is transient: $\bar\omega^2$ is an empirical-process concentration coefficient of order $1/\tau$, while $\Delta L^{(0)}/(\mu^\prime T)$ measures the initialization-dependent error. Thus, also in the nonlinear setting, the iterate-dependent coefficient can be made sufficiently small for any fixed $\kappa$ by having $\tau \geq \tau_{\textsf{burn-in}}$.

\begin{theorem}[Convergence Bound]
\label{thm:convergence_nl_mainb}
Suppose that the conditions of Lemma \ref{lem:bound_GT_nonlinear_main} hold. Then, with probability at least $1-\delta$, it holds that
\begin{align*}
\textcolor{black}{Q_T}&\textcolor{black}{\lesssim\frac{L^\prime\Delta L^{(0)}}{T}
+\left[\kappa\bar J^2+\bar\omega^2\left(\kappa+\frac{1}{|\cH|}\right)\right]\Gamma_{\textsf{nonlin}}}+\bar J^2\sigma^2\left(\kappa+\frac{1}{|\cH|}\right)
\left(\frac{\textcolor{black}{p}+\log(2|\cH|T/\delta)}{\tau}\right).
\end{align*}
\end{theorem}

\begin{proof} We provide the proof in Appendix \ref{appendix:nonlinear}.
\end{proof}

\noindent \textbf{Discussion of Theorem \ref{thm:convergence_nl_mainb}.} The convergence theorem demonstrates that the conclusion from the linear setting persists beyond a linear parameterization. The sample-size condition makes $\bar\omega^2(\kappa+1/|\cH|)/\mu^\prime$ small enough to absorb the second term of \eqref{G_T_Q_T_nonlinear} in the inequality of $Q_T$ when we conduct the ergodic convergence analysis. Therefore, for any fixed $\kappa =\mathcal{O}(f/n)$, with $f/n < 1/2$, a sufficiently large fresh batch removes the need for an additional upper bound on $\kappa$ as required in \cite{allouah2023tight}. The remaining terms consist of honest client's function heterogeneity bias and label noise term, where the later goes away with more data samples $\tau$ per client.

\begin{corollary}[Parameter Recovery Error]\label{corollary_nonlinear}
Suppose that the conditions of Theorem \ref{thm:convergence_nl_mainb} hold. Then, for every $\delta \in (0,1)$, with probability at least $1-\delta$, it holds
\begin{align*}
\frac{1}{T}\sum_{t=0}^{T-1}\frac{1}{|\cH|}\sum_{j\in\cH}
\mathbb{E}_X\Bigl[
\norm{h_{\theta^{(t)}}(X)-h_{\thetastar_j}(X)}^2
\Bigr]&\lesssim\frac{L^\prime \Delta L^{(0)}}{\mu^\prime T}
+\left(1+\frac{\bar J^2\kappa}{\mu^\prime} +\textcolor{black}{\frac{\bar\omega^2}{\mu^\prime}\left(\kappa+\frac{1}{|\cH|}\right)}\right)
\Gamma_{\textsf{nonlin}}\\
&+\frac{\bar J^2\sigma^2}{\mu^\prime}
\left(\kappa+\frac{1}{|\cH|}\right)
\left(\frac{\textcolor{black}{p}+\log(2|\cH|T/\delta)}{\tau}\right).
\end{align*}
\end{corollary}
\begin{proof} The proof is detailed in Appendix \ref{appendix:nonlinear}.
\end{proof}

\noindent \textbf{Discussion of Corollary \ref{corollary_nonlinear}.} The corollary translates approximate stationarity, i.e., the bound on $Q_T$, into recovery of the honest client functions. The term $\Gamma_{\textsf{nonlin}}$ is unavoidable: a single global model parameter $\theta \in \mathbb{R}^p$ cannot simultaneously characterize heterogeneous ground-truth functions. The additional price of adversarial robustness is explicit inflated through $\kappa$, while the statistical contribution vanishes with the number of fresh samples per client $\tau$. In particular, for any fixed $\kappa$, homogeneous client functions and increasing $\tau$ yield vanishing averaged prediction error at the optimization rate.

Taken together, the linear and nonlinear results demonstrate that the irreducible bias is determined by the heterogeneity among honest client models $\Gamma_{\textsf{lin}} < \infty$ and $\Gamma_{\textsf{nonlin}} < \infty$, while the finite-sample and initialization-dependent component of gradient heterogeneity can be reduced through the fresh-batch size $\tau$ and the number of iterations $T$, respectively.

\section{Related Work}

Finally, we position our results relative to federated optimization under heterogeneity and Byzantine-robust aggregation.\\

\noindent\textbf{FL under heterogeneity.} Client heterogeneity is a crucial difficulty in federated optimization. FedAvg and its variants reduce communication but may suffer from client drift when local objectives differ~\citep{mcmahan2017communication,gorbunov2021localsgd}. FedProx controls this drift through a proximal local objective \citep{li2020federated}, while SCAFFOLD uses control variates to correct the bias induced by heterogeneous local updates \citep{karimireddy2020scaffold}. Their analyses, and much of the subsequent literature, use bounded gradient assumptions to characterize client heterogeneity. They focus primarily on the optimization consequences of non-identical objectives in the absence of Byzantine behavior. Our focus is complementary: we study how the empirical gradient heterogeneity itself evolves under the statistical regression model and how this quantity propagates through the adversarially robust guarantees.\\

\noindent\textbf{Byzantine-robust distributed learning.}
A broad line of work designs aggregators that tolerate arbitrary client updates. Examples include Krum \citep{blanchard2017machine}, coordinate-wise median and trimmed mean \citep{yin2018byzantine}, geometric-median-based aggregation \citep{chen2017distributed,pillutla2022robust}, Bulyan \citep{mhamdi2018hidden}, and momentum or bucketing-based methods \citep{farhadkhani2022byzantine,karimireddy2020byzantine,allouah2023fixing}. These methods differ computationally and statistically, but their convergence analyses must all control the deviation between the aggregate and the mean update of the honest clients (Definition \ref{def:robust-agg}). Moreover, nearest neighbor mixing (NNM) \cite{allouah2023fixing} further pre-processes aggregators to have  $\kappa=\mathcal{O}(f/n)$. Pre-aggregation gradient clipping is also commonly used to reduce the effect of unusually large updates. \cite{allouah2025adaptive} demonstrate that a fixed clipping threshold need not preserve $(f,\kappa)$-robustness and propose adaptive clipping. Accordingly, our sample burn-in applies to any aggregator satisfying Definition~\ref{def:robust-agg}.

A recent line of work also studies high-dimensional statistical rates \citep{data2021byzantine,zhu2023byzantine}, gradient splitting \citep{liu2023gradient}, label skewness \citep{bao2024boba}, client subsampling and local updates \citep{allouah2024fedro}, and communication compression \citep{rammal2024compression} under Byzantine attacks. These works develop algorithms and guarantees for particular sources of heterogeneity. Our goal with this work is to expose the statistical structure of the gradient-heterogeneity term that appears in such robust federated optimization analyses, here with an emphasis to the regression problem.\\

\noindent\textbf{Gradient dissimilarity.} Uniform $G$-dissimilarity is standard in Byzantine-robust federated learning~\citep{karimireddy2020byzantine,data2021byzantine,allouah2023fixing}, while related bounded-dissimilarity conditions appear in federated optimization \citep{li2020federated,karimireddy2020scaffold}. However, a uniform $G$ bound may fail even for simple regression problems. \cite{allouah2023tight} address this limitation establishing tight lower and upper bounds for $(G,B)$-dissimilarity. Our analysis is built upon the observation that $G$-dissimilarity fails, but takes a different route: rather than assuming fixed constants $G$ and $B$, we derive the gradient heterogeneity from the regression model, sample size, and noise distribution, and shed the light to the interpretation of these constants.

For linear and nonlinear regression, we derive a trajectory-dependent bound in which the irreducible bias term $G$ is in the order of the pairwise honest client's model parameter or function heterogeneity, while the multiplicative term $B$ is in the order of a covariance concentration term. Such multiplicative coefficient scales inversely with $\tau$. The condition needed to close the convergence recursion is therefore an explicit sample-size burn-in, rather than a fixed constraint on $\kappa$, i.e., $\kappa B^2 < 1$ as required in \cite{allouah2023tight}.

\section{Conclusions and Future Work} \label{sec:conclusions}

We studied the gradient heterogeneity dynamics of adversarially robust federated regression. Rather than assuming a uniform dissimilarity bound with a priori selected constants $G$ and $B$, as in $(G,B)$-dissimilarity, we derived the honest-gradient heterogeneity from the statistical data-generating model. For linear and nonlinear regression, the resulting bound separates model parameter heterogeneity, finite-sample noise, and an initialization-dependent term. This yields non-asymptotic convergence and recovery bounds in which the restriction on the robust aggregation coefficient imposed in prior work \cite{allouah2023tight} is replaced by an explicit sample burn-in condition on $\tau$. 

Future work would consider a fixed local dataset over iterations and would control the dependence between the iterates and reused samples in the analysis. Moreover, extending the analysis to account for local model updates, partial client participation, and time-varying adversarial sets would bring the theory closer to practical federated systems. 

\section{Acknowledgments} 
Leonardo F. Toso is funded by the Center for AI and Responsible Financial Innovation (CAIRFI) Fellowship and the Columbia Presidential Fellowship. James Anderson is partially funded by NSF grants EECS 2144634 and CNS 2535097 and the Center of AI Technology (CAIT) in collaboration with Amazon.

\bibliographystyle{alpha}
\bibliography{references}

\newpage
\appendix

\section{Appendix Roadmap}
\label{appendix:roadmap}

The appendix is organized as follows. Appendix~\ref{appendix:linear} reports the linear-regression analysis. We first recall the regression model and the gradients in Section~\ref{subsec:linear_model}, and collect the supporting inequalities and concentration inequalities in Section~\ref{subsec:linear_supporting_inequalities}. Sections~\ref{subsec:conc} and~\ref{subsec:noise} establish the concentration and noise bounds used throughout the proofs. We then derive the gradient heterogeneity decomposition in Section~\ref{subsec:grad_diff}, the optimization error recursion in Section~\ref{subsec:error_dyn}, and the convergence guarantee in Section~\ref{subsec:convergence_linear}. Finally,~Section~\ref{subsec:param_recovery_linear_fresh} provides the resulting non-asymptotic parameter recovery error bounds.

Appendix~\ref{appendix:nonlinear} reports the nonlinear-regression analysis. Section~\ref{subsec:nonlinear_model} recalls the nonlinear regression model and its gradients. Sections~\ref{subsec:nl_grad_diff}, \ref{subsec:nonlinear_concentration}, and~\ref{subsec:nonlinear_heterogeneity} derive the gradient heterogeneity bound, while in Section~\ref{subsec:nonlinear_convergence} we establish the bound on the gradient heterogeneity and the corresponding convergence guarantee. Section~\ref{subsec:param_recovery_nonlinear}
concludes with the non-asymptotic function-recovery error bound.

\section{Linear Regression}
\label{appendix:linear}

We consider the parametric linear regression problem and derive non-asymptotic parameter-recovery guarantees for adversarially robust federated learning. Our analysis will combine a first-principles derivation of the per-iteration gradient heterogeneity with the parameter-optimization error recursion. We also establish a non-asymptotic ergodic convergence bound and compare the conditions on the robust aggregation coefficient $\kappa$ obtained in our analysis to those in \cite{allouah2023tight}.
 
\subsection{Regression Model and Gradients}
\label{subsec:linear_model}
 
We recall that each honest client $i\in\cH$ has an unknown ground-truth parameter $\thetastar_i \in \R^{\vdy\times\vd}$ and is assumed to generate data according to
\begin{align*}
Y_{i,k} = \thetastar_i X_{i,k} + V_{i,k} \quad \text{ for all } k=1,\dots,\tau,
\end{align*}
where $X_{i,k}\in\R^{\vd}$ satisfies $\norm{X_{i,k}} \leq R$ almost surely. 
Here $V_{i,k}\in\R^{\vdy}$ denotes a vector-valued label noise. We also recall that the empirical squared loss for client $i \in \mathcal{H}$ is given by
\begin{align*}
\hat L_i(\theta) := \frac{1}{\tau}\sum_{k=1}^{\tau}\norm{Y_{i,k} - \theta X_{i,k}}^2.
\end{align*}
At a fixed iteration $t$, we suppress the iteration index on the fresh samples and define the empirical covariance and noise terms as follows:
\begin{align*}
\hSigma_i := \frac{1}{\tau}\sum_{k=1}^{\tau}X_{i,k}X_{i,k}^\top \in \R^{\vd\times\vd} \text{ and } \xi_i := \frac{2}{\tau}\sum_{k=1}^{\tau}V_{i,k}X_{i,k}^\top \in \R^{\vdy\times\vd}.
\end{align*}
 We then have the following gradient expression:
\begin{align}\label{eq:grad_linear}
\nabla_\theta \hat  L_i^{(t)}(\theta) = 2(\theta - \thetastar_i)\hSigma_i - \xi_i.
\end{align} 
 
We note that, under the isotropic data assumption $\Sigma_i = I_{\vd}$ (Assumption \ref{assumption:isotropic}), and by expanding the gradient at iteration $t \in \{0,1,\ldots, T-1\}$, we obtain
\begin{align}\label{eq:grad_expanded}
\nabla_\theta \hat  L_i^{(t)}(\theta^{(t)}) = 2(\theta^{(t)} - \thetastar_i)\textcolor{black}{\Sigma_i}  + 2(\theta^{(t)}-\thetastar_i)(\hSigma_i - \Sigma_i) - \xi_i.
\end{align}
 
Then, the honest-average gradient is as follows:
\begin{align}\label{eq:avg_grad}
\nabla_\theta \hat  L_\cH^{(t)}(\theta^{(t)})
= \frac{2}{|\cH|}\sum_{i\in\cH}(\theta^{(t)}-\thetastar_i)\hSigma_i - \barxi, \text{ with the honest-average noise term  }
\barxi := \frac{1}{|\cH|}\sum_{i\in\cH}\xi_i.
\end{align}
 
\begin{remark}[$G$-gradient dissimilarity fails for linear regression] \label{remark:G_fails}
In general, $G$-gradient dissimilarity \cite{karimireddy2020byzantine} does not hold for the linear regression model \eqref{eq:linear_model} under heterogeneous client parameters. Indeed, from \eqref{eq:grad_linear} and the isotropic data assumption, we have a per-iteration gradient heterogeneity bound for
\begin{align*}
G^{(t)} = \frac{1}{|\cH|}\sum_{i\in\cH}
\norm{\nabla_\theta L_i(\theta^{(t)}) - \nabla_\theta L_\cH(\theta^{(t)})}^2,
\end{align*}
that has a term of order
$\frac{1}{|\cH|}\sum_{j\in\cH} \| \thetastar_j - \thetastar_i \|^2$ from the model parameter heterogeneity, and additional terms depending on $\|\theta^{(t)} - \thetastar_j\|^2$
through the covariance deviations $\|\hSigma_i - \Sigma_i\|^2$. As $\theta^{(t)}$ may move away from the optimum, the latter terms grow unboundedly, and thus no finite uniform $G^2$ can bound $G^{(t)}$ for all $\theta^{(t)}$.  This is consistent
with \citep{allouah2023tight}, that discusses this limitation of $G$-dissimilarity for a two-client, one-dimensional motivating example.
\end{remark}

\subsection{Supporting Inequalities}
\label{subsec:linear_supporting_inequalities}

We also list some standard supporting inequalities that are leveraged throughout the derivations.

\begin{lemma}[Cauchy-Schwarz]
\label{lem:cs}
For any nonnegative random variable $U$, $\mathbb{E}[U]\leq \sqrt{\mathbb{E}[U^2]}$.
\end{lemma}

\begin{lemma}[Young's inequality]
\label{lem:young}
For any $a,b \in \mathbb{R}$ and any $\gamma>0$,
\begin{align*}
2ab \leq \gamma a^2 + \frac{1}{\gamma} b^2.
\end{align*}
In addition, for any $u,v \in \mathbb{R}$, it is also equivalent to write
\begin{align*}
(u+v)^2
\leq
(1+\gamma)u^2 + \left(1+\frac{1}{\gamma}\right)v^2.
\end{align*}
\end{lemma}

\begin{lemma}
\label{lem:variance_decomposition}
Let $a_1, \dots, a_n \in \mathbb{R}^d$ and define their average
$$
\bar a := \frac{1}{n}\sum_{i=1}^n a_i.
$$
Then, the following identity holds:
\begin{align*}
\frac{1}{n}\sum_{i=1}^n \|a_i - \bar a\|^2 = \frac{1}{n}\sum_{i=1}^n \|a_i\|^2 - \|\bar a\|^2.
\end{align*}
\end{lemma}

\begin{lemma}[Jensen's inequality]
\label{lem:jensen}
Let $X$ be a random variable and let $\phi$ be a convex function such that the expectations below are well defined. Then,
\begin{align*}
\phi\left(\mathbb{E}[X]\right)
\leq \mathbb{E}\left[\phi(X)\right].
\end{align*}
As a special case, for any random vector $X \in \mathbb{R}^d$,
\begin{align*}
\left\|\mathbb{E}[X]\right\|^2 \leq \mathbb{E}\left[\|X\|^2\right].
\end{align*}
\end{lemma}

\subsection{Supporting Concentration Inequalities for Linear Regression}
\label{subsec:conc}
 
We now state standard concentration inequalities that are leveraged throughout our derivations. 

\begin{lemma}[Matrix Bernstein - Theorem 1.4 from \cite{tropp2012user}] \label{lemma:Bernstein}
Consider a finite sequence $\{Z_k\}_k$ of independent, random, self-adjoint matrices with dimension $d$. We assume that each random matrix satisfies
$$
\textcolor{black}{\mathbb{E}Z_k = 0} \text{ and }
\textcolor{black}{\lambda_{\max}(Z_k) \leq L}
\text{ almost surely}.
$$
Then, for all $t \geq 0$,
$$
\mathbb{P}\left\{
\lambda_{\max}\left(\sum_k \textcolor{black}{Z_k}\right) \geq t
\right\} \leq d 
\exp\left( \frac{-t^2/2}{\textcolor{black}{v}+\textcolor{black}{L}t/3}
\right),
$$
where \textcolor{black}{$v := \left\| \sum_k \mathbb{E}\left(Z_k^2\right) \right\|$}.
\end{lemma}

\begin{lemma}[Average of sub-Gaussian random matrices]
\label{lem:subg_mean}
Let $M_1,\dots,M_m\in\R^{p\times q}$ be independent, mean-zero random matrices
satisfying
$$
\norm{M_j}_{\psi_2} := \sup_{U\in\mathbb{S}_F^{p\times q}}\norm{\langle U,M_j\rangle}_{\psi_2} \leq K, \forall j=1,\dots,m.
$$
Then there exists a  constant $C_1>0$ such that, for every
$\delta\in(0,1)$, with probability at least $1-\delta$, the following bound holds:
\begin{align*}
\left\|\frac{1}{m}\sum_{j=1}^m M_j\right\|_F \leq
C_1 K\sqrt{\frac{pq+\log(1/\delta)}{m}}.
\end{align*}
\end{lemma}
\begin{proof}
This result follows by treating each matrix $M_j\in\mathbb{R}^{p\times q}$ as a vector in $\mathbb{R}^{pq}$. For any fixed $U\in\mathbb{S}_F^{p\times q}$, the random variables $\langle U,M_j\rangle_F$ are independent, mean-zero, and sub-Gaussian with $\psi_2$-norm bounded by $K$. Therefore, the sub-Gaussian Hoeffding inequality \citep[Theorem 2.7.3]{vershynin2018high} yields the following probability bound 

\begin{align}\label{eq:sub-Gaussian_Hoeffding}
\mathbb{P}\left(\left|\left\langle U,\frac{1}{m}\sum_{j=1}^m M_j\right\rangle_F\right|\ge t\right)\le 2\exp\left(-\frac{cmt^2}{K^2}\right).    
\end{align}

Next, let $\mathcal{N}_\xi$ be a $\xi$-net \citep[Definition 4.2.1]{vershynin2018high} of the Frobenius unit sphere. A standard volumetric estimate yields $|\mathcal{N}_\xi|\le \left(1+\frac{2}{\xi}\right)^{pq}$, while the net approximation lemma given in \citep[Lemma 4.4.1]{vershynin2018high} provides 

$$\left\|\frac{1}{m}\sum_{j=1}^m M_j\right\|_F\le \frac{1}{1-\xi}\max_{U\in\mathcal{N}_\xi}\left|\left\langle U,\frac{1}{m}\sum_{j=1}^m M_j\right\rangle_F\right|$$. 

Hence, by taking a union bound over $\mathcal{N}_\xi$, with $\xi = \frac{1}{2}$, and by selecting $t\asymp K\sqrt{\frac{pq+\log(1/\delta)}{m}}$, with probability at least $1-\delta$, we have $\left\|\frac{1}{m}\sum_{j=1}^m M_j\right\|_F\le C_1 K\sqrt{\frac{pq+\log(1/\delta)}{m}}$.
\end{proof}
 
\begin{lemma}[Covariance concentration]
\label{lem:cov_conc}
Suppose that the isotropic data assumption (i.e., $\Sigma_i = I_{\vd}, \forall i \in \cH$ in Assumption \ref{assumption:isotropic}) holds. Then, for every $\delta\in(0,1)$, with probability at least $1-\delta$, simultaneously for all $i\in\cH$ and $t=0,\ldots,T-1$, we have
\begin{align}\label{eq:cov_conc}
\norm{\hSigma_i - \Sigma_i}_2 \leq
2(R^2+1) \left(\sqrt{\frac{\log(2|\mathcal{H}|T\vd/\delta)}{\tau}}
+ \frac{\log(2|\mathcal{H}|T\vd/\delta)}{\tau}\right).
\end{align}
\end{lemma}
 
\begin{proof}
We begin by defining the mean-zero self-adjoint matrices $Z_{i,k} := X_{i,k}X_{i,k}^\top
- I_{\vd}$, such that $\hSigma_i - \Sigma_i = \frac{1}{\tau}\sum_{k=1}^\tau
Z_{i,k}$. The spectral norm of each $Z_{i,k}$ is bounded as $\norm{Z_{i,k}}_2 \leq R^2 + 1$. We also define the variance $v = \norm{\sum_k\mathbb{E}[Z_{i,k}^2]}_2$. Note that since $Z_{i,k}^2 \preceq (R^4+2R^2+1)I_{\vd}$, we obtain $v \leq (R^2+1)^2\tau$. Therefore, applying the matrix Bernstein inequality in Lemma \ref{lemma:Bernstein} (see also \citep[Theorem 1.4]{tropp2012user}) to $Z_{i,k}$ and $-Z_{i,k}$ and union bounding over honest clients $i \in \mathcal{H}$ and iterations $t = 0, \ldots, T-1$, yield the expression in \eqref{eq:cov_conc}.
\end{proof}
 
\subsection{Bound for the Noise}
\label{subsec:noise}

We now turn our attention to bound the $\sigma$-sub-Gaussian noise term. 
 
\begin{lemma}
\label{lem:noise_bound}
Suppose that Assumption \ref{assm:subg} and $\norm{X_{i,k}}\leq R$ hold almost surely. Then there exists a  constant $C_1>0$ such that, for every
$\delta\in(0,1)$, with probability at least $1-\delta$, simultaneously for all honest clients and iterations, we have
\begin{align*}
\norm{\xi_i} + \norm{\barxi} \leq 
C_1 R\sigma \sqrt{\frac{\vd \vdy + \log(2|\cH|T/\delta)}{\tau}}
\left(1 + \frac{1}{\sqrt{|\cH|}}\right).
\end{align*}
Then, by triangle inequality,
$\norm{\xi_i - \barxi} \leq \norm{\xi_i} + \norm{\barxi}$, which provides the same bound.
\end{lemma}
 
\begin{proof}
We apply Lemma~\ref{lem:subg_mean} to the matrices 
$\xi_i = \frac{2}{\tau}\sum_{k=1}^\tau V_{i,k}X_{i,k}^\top\in\R^{\vdy\times\vd}$, for all $i \in \mathcal{H}$. For this, let us fix $U\in\mathbb{S}_F^{\vdy\times\vd}$ and write
$$
\langle U,\xi_i\rangle
= \frac{2}{\tau}\sum_{k=1}^\tau V_{i,k}^\top U X_{i,k}.
$$
Note that as $\norm{UX_{i,k}}_2\leq\norm{U}_F\norm{X_{i,k}}_2\leq R$, the scalar $V_{i,k}^\top UX_{i,k}$ is mean-zero sub-Gaussian with $\norm{V_{i,k}^\top UX_{i,k}}_{\psi_2}\leq R\sigma$. By the Hoeffding inequality for sub-Gaussian random variables \eqref{eq:sub-Gaussian_Hoeffding} (see also \citep[Theorem 2.2.1]{vershynin2018high}), we obtain
$\norm{\xi_i}_{\psi_2}\leq 2C R\sigma/\sqrt{\tau}$ for some $C>0$. Therefore, applying Lemma \ref{lem:subg_mean} with $m=\tau$, $p=\vdy$, and $q=\vd$, we obtain the following bound:

\begin{align}\label{eq:xi_bound}
\norm{\xi_i}_F \leq C_1 R\sigma\sqrt{\frac{\vd \vdy+\log(1/\delta)}{\tau}}.
\end{align}

In addition, as $\{\xi_j\}_{j\in\cH}$ are independent across honest clients, $\barxi$ is an average of $|\cH|$ independent mean-zero sub-Gaussian matrices with the same
sub-Gaussian norm. Therefore, Lemma \ref{lem:subg_mean} provides
\begin{align}\label{eq:xibar_bound}
\norm{\barxi}_F \leq C_1 R\sigma\sqrt{\frac{\vd\vdy+\log(1/\delta)}{\tau|\cH|}}.
\end{align}
The proof is then completed by combining \eqref{eq:xi_bound}-\eqref{eq:xibar_bound} and absorbing any constants in $C_1$.
\end{proof}
 
\subsection{Gradient Heterogeneity Decomposition}
\label{subsec:grad_diff}
 
We begin by subtracting \eqref{eq:avg_grad} from \eqref{eq:grad_expanded} and using the isotropic data assumption $\Sigma_i = I_{\vd}$ for all $i \in \mathcal{H}$ (Assumption \ref{assumption:isotropic}) as follows: 
\begin{align*}
\nabla_\theta \hat  L_i^{(t)}(\theta^{(t)}) - \nabla_\theta \hat  L_\cH^{(t)}(\theta^{(t)})
&= \frac{2}{|\cH|}\sum_{j\in\cH}(\thetastar_j - \thetastar_i)\textcolor{black}{\Sigma_i} - (\xi_i - \barxi) \\
& + \frac{2}{|\cH|}\sum_{j\in\cH}(\theta^{(t)}-\thetastar_j)(\hSigma_i - \Sigma_i) \\
& + \frac{2}{|\cH|}\sum_{j\in\cH}(\theta^{(t)}-\thetastar_j)(\Sigma_j - \hSigma_j) \\
& + \frac{2}{|\cH|}\sum_{j\in\cH}(\thetastar_j - \thetastar_i)(\hSigma_i - \Sigma_i).
\end{align*}
 
Therefore, by taking norms and applying the triangle inequality, we obtain
\begin{align}\label{eq:grad_diff_norm}
\norm{\nabla_\theta \hat  L_i^{(t)}(\theta^{(t)}) - \nabla_\theta \hat  L_\cH^{(t)}(\theta^{(t)})}
&\leq \underbrace{\frac{2}{|\cH|}\sum_{j\in\cH}\norm{\thetastar_j - \thetastar_i}}_{\text{model heterogeneity}}
+ \underbrace{\norm{\xi_i - \barxi}}_{\text{noise}} \notag\\
&+ \underbrace{\frac{2}{|\cH|}\sum_{j\in\cH}\norm{\Sigma_i - \hSigma_i}\norm{\theta^{(t)}-\thetastar_j}}_{\text{covariance concentration (client } i)}  \notag\\
&+ \underbrace{\frac{2}{|\cH|}\sum_{j\in\cH}\norm{\Sigma_j - \hSigma_j}\norm{\theta^{(t)}-\thetastar_j}}_{\text{covariance concentration (client } j)} \notag\\
&+ \underbrace{\frac{2}{|\cH|}\sum_{j\in\cH}\norm{\Sigma_i - \hSigma_i}\norm{\thetastar_j - \thetastar_i}}_{\text{cross term}}.
\end{align}
 
\noindent \textbf{Burn-in time condition.}
Fix $\delta\in(0,1)$ and, for some $\omega>0$, let
\begin{align}\label{omega2}
\omega^2:=C_1(R^2+1)^2\textcolor{black}{\left(\frac{\log(2|\cH|T(\vd+\vdy)/\delta)}{\tau}\right)}.
\end{align}
\begin{align}\label{eq:tau_condition}
\tau\geq \textcolor{black}{C_0C_1}(R^2+1)^2
\max\left\{1,\kappa+\frac{1}{|\cH|}\right\}
\textcolor{black}{\log\left(\frac{2|\cH|T(\vd+\vdy)}{\delta}\right)}.
\end{align}
\textcolor{black}{Here, $C_0>0$ is a sufficiently large constant. Thus, \eqref{eq:tau_condition} guarantees $\omega\leq1$ and
\begin{align*}
\omega^2\left(\kappa+\frac{1}{|\cH|}\right)\leq\frac{1}{C_0},
\end{align*}
where $C_0$ can be chosen large enough to absorb all constants appearing below.}

Therefore, under the burn-in time condition \eqref{eq:tau_condition}, Lemma \ref{lem:cov_conc} yields
$\norm{\hSigma_i - \Sigma_i} \leq \frac{\omega}{2}$ and $\norm{\hSigma_j - \Sigma_j} \leq \frac{\omega}{2}$, for all $i,j \in \mathcal{H}$. Then the covariance concentration and cross terms in \eqref{eq:grad_diff_norm} are absorbed as follows:
\begin{align}\label{eq:grad_diff_simplified}
\norm{\nabla_\theta \hat L_i(\theta^{(t)}) - \nabla_\theta \hat L_\cH(\theta^{(t)})} \leq 
\frac{3}{|\cH|}\sum_{j\in\cH}\norm{\thetastar_j - \thetastar_i}+ \norm{\xi_i - \barxi}
+ \frac{\textcolor{black}{2}\omega}{|\cH|}\sum_{j\in\cH}\norm{\theta^{(t)} - \thetastar_j}. 
\end{align}
where we consider $\omega \leq 1$ for the first term. We reemphasize that, as discussed in Remark \ref{remark:G_fails}, the gradient heterogeneity bound has a model-heterogeneity term of order $\frac{1}{|\cH|}\sum_{j\in\cH}\norm{\thetastar_j-\thetastar_i}$, as well as an additional iterate-dependent term that scales as $\frac{1}{|\cH|}\sum_{j\in\cH}\norm{\theta^{(t)}-\thetastar_j}$. Therefore, if $\theta^{(t)}$ moves arbitrarily far from the optimum $\thetastar_j$ of any honest client $j\in\cH$, then even in this simple linear regression setting, the gradient heterogeneity bound need not remain uniformly bounded.

To complete the bound for $G^{(t)}$, we square both sides of \eqref{eq:grad_diff_simplified}, apply $\norm{a+b+c}^2\leq 3\norm{a}^2+3\norm{b}^2+3\norm{c}^2$, and average over honest clients $i\in\cH$. Therefore, the per-iteration gradient heterogeneity satisfies 

\begin{align}\label{eq:Gt_linear}
G^{(t)} &\leq \underbrace{\frac{27}{|\cH|^2}\sum_{i\in\cH}\sum_{j\in\cH}
\norm{\thetastar_j - \thetastar_i}^2}_{= 27 \Gamma_{\textsf{lin}}} + 3C_1 R^2\sigma^2\left(\frac{\vd \vdy + \log(2|\mathcal{H}|/\delta)}{\tau} \right)\left(1+\frac{1}{|\cH|}\right) + \frac{3\omega^2}{|\cH|}\sum_{j\in\cH}\norm{\theta^{(t)} -\thetastar_j}^2,
\end{align}
where we recall that the model heterogeneity as follows:
\begin{align*}
\Gamma_{\textsf{lin}} := \frac{1}{|\cH|^2}\sum_{i\in\cH}\sum_{j\in\cH}
\norm{\thetastar_i - \thetastar_j}^2.
\end{align*}

\subsection{Optimization Error Recursion}
\label{subsec:error_dyn}

We now relate the parameter error, i.e., $\norm{\theta^{(t)} -\thetastar_j}$, directly to the gradient of the population loss. Let
\begin{align*}
\bar\theta^\star:=\frac{1}{|\cH|}\sum_{i\in\cH}\thetastar_i.
\end{align*}
We note that under isotropy and conditional mean-zero noise, we have 
\begin{align}\label{eq:population_linear_gradient}
\nabla L_\cH(\theta)=2(\theta-\bar\theta^\star).
\end{align}
Moreover, the variance decomposition in Lemma \ref{lem:variance_decomposition} provides the following identity.
\begin{align}\label{eq:Delta_Q_relation}
\Delta^{(t)}:=\frac{1}{|\cH|}\sum_{j\in\cH}\norm{\theta^{(t)}-\thetastar_j}^2
=\frac{1}{4}\norm{\nabla L_\cH(\theta^{(t)})}^2+\frac{1}{2}\Gamma_{\textsf{lin}}.
\end{align}
We then substitute \eqref{eq:Delta_Q_relation} into \eqref{eq:Gt_linear} to obtain
\begin{align}\label{eq:Gt_linear_population}
G^{(t)}\lesssim \Gamma_{\textsf{lin}}
+\omega^2\norm{\nabla L_\cH(\theta^{(t)})}^2
+R^2\sigma^2\left(\frac{\vd\vdy+\log(2|\cH|T/\delta)}{\tau}\right),
\end{align}
where all constants are adsorbed into $\lesssim$.

Let us also define the honest-average gradient error
\begin{align}\label{eq:Zt_linear}
Z^{(t)}:=\norm{\nabla\hat  L_\cH^{(t)}(\theta^{(t)})
-\nabla L_\cH(\theta^{(t)})}^2.
\end{align}
To bound this quantity, we first subtract the population gradient \eqref{eq:population_linear_gradient} from the honest-average fresh-batch gradient. This yields the following expression:
\begin{align*}
\nabla\hat  L_\cH^{(t)}(\theta^{(t)})-\nabla L_\cH(\theta^{(t)})
&=\frac{2}{|\cH|}\sum_{i\in\cH}
(\theta^{(t)}-\thetastar_i)(\hat \Sigma_i-I_\vd)-\bar\xi.
\end{align*}
Therefore, using $\norm{a+b}^2\leq2\norm{a}^2+2\norm{b}^2$, we obtain
\begin{align*}
Z^{(t)}
&\leq 8\norm{\frac{1}{|\cH|}\sum_{i\in\cH}
(\theta^{(t)}-\thetastar_i)(\hat \Sigma_i-I_\vd)}^2
+2\norm{\bar\xi}^2.
\end{align*}

{\color{black}
We proceed to bound the first term. To do so, let us define the following quantities $A_i^{(t)}:=\theta^{(t)}-\thetastar_i$ and recall from Lemma \ref{lem:cov_conc} that $Z_{i,k}:=X_{i,k}X_{i,k}^\top-I_\vd$, and $M_{i,k}^{(t)}:=A_i^{(t)}Z_{i,k}$.
Then, we can write 
\begin{align*}
\frac{1}{|\cH|}\sum_{i\in\cH}A_i^{(t)}(\widehat\Sigma_i-I_\vd)=\frac{1}{|\cH|\tau}\sum_{i\in\cH}\sum_{k=1}^{\tau}M_{i,k}^{(t)}.
\end{align*}
We then recall that the fresh-sample assumption makes the matrices $M_{i,k}^{(t)}$ independent conditionally on $\theta^{(t)}$, while the data isotropy assumption implies
\begin{align*}
\mathbb E[Z_{i,k}\mid\theta^{(t)}]=0 \text{ and }
\mathbb E[M_{i,k}^{(t)}\mid\theta^{(t)}]=0.
\end{align*}
We can also define the self-adjoint matrix
\begin{align*}
\mathcal S(M_{i,k}^{(t)}):=
\begin{pmatrix}
0&M_{i,k}^{(t)}\\
M_{i,k}^{(t)\top}&0
\end{pmatrix},
\end{align*}
where we have $\|\mathcal S(M_{i,k}^{(t)})\|=\|M_{i,k}^{(t)}\|$. As we assume that $\|X_{i,k}\|\leq R$, we have
\begin{align*}
\|Z_{i,k}\|\leq R^2+1 \text{ and }
\|M_{i,k}^{(t)}\|\leq (R^2+1)\|A_i^{(t)}\|.
\end{align*}
In addition, note that we can write
\begin{align*}
M_{i,k}^{(t)}(M_{i,k}^{(t)})^\top
&=A_i^{(t)}Z_{i,k}^2(A_i^{(t)})^\top,\\
(M_{i,k}^{(t)})^\top M_{i,k}^{(t)}
&=Z_{i,k}(A_i^{(t)})^\top A_i^{(t)}Z_{i,k},
\end{align*}
which implies that both diagonal blocks of
$\mathbb E[\mathcal S(M_{i,k}^{(t)})^2\mid\theta^{(t)}]$
have norm at most $(R^2+1)^2\|A_i^{(t)}\|^2$. Thus, we can write
\begin{align*}
v_t&:=\left\|\sum_{i\in\cH}\sum_{k=1}^{\tau}
\mathbb E\left[\mathcal S(M_{i,k}^{(t)})^2
\mid\theta^{(t)}\right]\right\|\leq  \tau b^2\sum_{i\in\cH}\|A_i^{(t)}\|^2.
\end{align*}
Let us then set $\ell:=\log\left(\frac{2|\cH|T(\vd+\vdy)}{\delta}\right).$ We are now ready to Lemma~\ref{lemma:Bernstein} and take union bound over $t=0,\ldots,T-1$ to obtain
\begin{align*}
&\left\|\frac{1}{|\cH|}\sum_{i\in\cH}
A_i^{(t)}(\hat\Sigma_i-I_\vd)\right\|\lesssim b\sqrt{\frac{\ell}{|\cH|\tau}
\left(\frac{1}{|\cH|}\sum_{i\in\cH}\|A_i^{(t)}\|^2\right)}
+b\frac{\ell}{|\cH|\tau}\max_{i\in\cH}\|A_i^{(t)}\|.
\end{align*}
Moreover, we have
\begin{align*}
\max_{i\in\cH}\|A_i^{(t)}\|&\leq\left(\sum_{i\in\cH}\|A_i^{(t)}\|^2\right)^{1/2}=\sqrt{|\cH|\Delta^{(t)}}.
\end{align*}
We then note that under the burn-in condition~\eqref{eq:tau_condition}, the second term is dominated by the first. Therefore, we have
\begin{align*}
\left\|\frac{1}{m}\sum_{i\in\cH}
A_i^{(t)}(\hat\Sigma_i-I_\vd)\right\|
&\lesssim b\sqrt{\frac{\ell}{|\cH|\tau}\Delta^{(t)}}.
\end{align*}
By squaring this inequality and using the definition of $\omega^2$ \eqref{omega2}, we obtain
\begin{align*}
\left\|\frac{1}{|\cH|}\sum_{i\in\cH}
(\theta^{(t)}-\thetastar_i)(\hat\Sigma_i-I_\vd)\right\|^2
&\lesssim\frac{(R^2+1)^2\ell}{\tau|\cH|}\Delta^{(t)}\lesssim\frac{\omega^2}{|\cH|}\Delta^{(t)}.
\end{align*}
}

In addition, we recall that Lemma \ref{lem:noise_bound} implies
\begin{align*}
\norm{\bar\xi}^2
\lesssim R^2\sigma^2
\left(\frac{\vd\vdy+\log(2T/\delta)}{\tau|\cH|}\right),
\end{align*}
Therefore, combining these bounds and substituting \eqref{eq:Delta_Q_relation}, we obtain
\begin{align}\label{eq:Zt_linear_bound}
Z^{(t)}\lesssim\frac{\omega^2}{|\cH|}
\norm{\nabla L_\cH(\theta^{(t)})}^2
+\frac{\omega^2}{|\cH|}\Gamma_{\textsf{lin}}
+R^2\sigma^2\left(\frac{\vd\vdy+\log(2T/\delta)}{\tau|\cH|}\right).
\end{align}

\subsection{Convergence Analysis}
\label{subsec:convergence_linear}

Let $e^{(t)}:=\sF^{(t)}-\nabla L_\cH(\theta^{(t)})$. By $(f,\kappa)$-robustness (Definition \ref{def:robust-agg}) and \eqref{eq:Zt_linear}, we have 
\begin{align}\label{eq:agg_error_linear}
\norm{e^{(t)}}^2\leq2\kappa G^{(t)}+2Z^{(t)}.
\end{align}
The population loss is $2$-smooth under isotropy. Thus, for $\eta\leq1/2$, we have 
\begin{align}\label{eq:descent_linear_population}
L_\cH(\theta^{(t+1)})-L_\cH(\theta^{(t)})
\leq-\frac{\eta}{2}\norm{\nabla L_\cH(\theta^{(t)})}^2
+\frac{\eta}{2}\norm{e^{(t)}}^2.
\end{align}

\begin{theorem}\label{thm:convergence_linear_fresh}
Suppose that Assumptions \ref{assumption:isotropic} and \ref{assm:subg} hold, $\eta\leq1/2$, and the burn-in condition \eqref{eq:tau_condition} holds. Then, for every $\delta\in(0,1)$, with probability at least $1-\delta$,
\begin{align*}
Q_T:=\frac{1}{T}\sum_{t=0}^{T-1}\norm{\nabla L_\cH(\theta^{(t)})}^2
\lesssim\frac{\Delta L^{(0)}}{\eta T}
+\left(\kappa+\frac{1}{|\cH|}\right)\Gamma_{\textsf{lin}}
+R^2\sigma^2\left(\kappa+\frac{1}{|\cH|}\right)
\left(\frac{\vd\vdy+\log(2|\cH|T/\delta)}{\tau}\right).
\end{align*}
\end{theorem}

\begin{proof}
We begin by summing \eqref{eq:descent_linear_population} over $t=0,\ldots,T-1$ to write
\begin{align*}
\frac{\eta}{2}\sum_{t=0}^{T-1}
\norm{\nabla L_\cH(\theta^{(t)})}^2
&\leq L_\cH(\theta^{(0)})-L_\cH(\theta^{(T)})
+\frac{\eta}{2}\sum_{t=0}^{T-1}\norm{e^{(t)}}^2\\
&\leq \Delta L^{(0)}
+\frac{\eta}{2}\sum_{t=0}^{T-1}\norm{e^{(t)}}^2,
\end{align*}
where we recall that $\Delta L^{(0)}:=L_\cH(\theta^{(0)})-L_\cH(\thetastar)$. The second inequality uses the optimality of $\thetastar$. Then, dividing by $\eta T/2$ and applying \eqref{eq:agg_error_linear}, we obtain
\begin{align*}
Q_T&\leq\frac{2\Delta L^{(0)}}{\eta T}
+\frac{1}{T}\sum_{t=0}^{T-1}\norm{e^{(t)}}^2 \leq\frac{2\Delta L^{(0)}}{\eta T}
+2\kappa G_T+2Z_T,
\end{align*}
where $Z_T:=\frac{1}{T}\sum_{t=0}^{T-1}Z^{(t)}$. By averaging \eqref{eq:Gt_linear_population} and \eqref{eq:Zt_linear_bound} over the iterations gives
\begin{align}\label{comparison_GB}
G_T
&\lesssim\Gamma_{\textsf{lin}}
+R^2\sigma^2\left(\frac{\vd\vdy+\log(2|\cH|T/\delta)}{\tau}\right)+\omega^2Q_T,\notag\\
Z_T
&\lesssim\frac{\omega^2}{|\cH|}Q_T
+\frac{\omega^2}{|\cH|}\Gamma_{\textsf{lin}}
+R^2\sigma^2\left(\frac{\vd\vdy+\log(2|\cH|T/\delta)}{\tau|\cH|}\right).
\end{align}
Therefore, by substituting these two bounds into the preceding inequality, we obtain
\begin{align*}
Q_T&\lesssim\frac{\Delta L^{(0)}}{\eta T}
+\left(\kappa+\frac{\omega^2}{|\cH|}\right)\Gamma_{\textsf{lin}}+R^2\sigma^2\left(\kappa+\frac{1}{|\cH|}\right)
\left(\frac{\vd\vdy+\log(2|\cH|T/\delta)}{\tau}\right)
+\omega^2\left(\kappa+\frac{1}{|\cH|}\right)Q_T.
\end{align*}
By \eqref{eq:tau_condition} we have \textcolor{black}{$\omega^2(\kappa+1/|\cH|)\leq1/C_0$}, and choosing $C_0$ sufficiently large makes the coefficient of $Q_T$ strictly smaller than one. We can therefore move the final term to the left-hand side and absorb the resulting numerical factor into the implicit constant absorbed into $\lesssim$. This proves the stated bound.
\end{proof}

\subsection{Non-Asymptotic Parameter Recovery Error Bound}
\label{subsec:param_recovery_linear_fresh}

We now state the recovery guarantee separately from the convergence analysis. We first note that under data isotropy, parameter error and prediction error coincide because, for any $\theta \in \R^{\vdy\times\vd}$ we can write 
\begin{align*}
\mathbb{E}_X\left[\norm{(\theta-\thetastar_j)X}^2\right]
=\operatorname{Tr}\left((\theta-\thetastar_j)^\top
(\theta-\thetastar_j)\mathbb{E}[XX^\top]\right)
=\norm{\theta-\thetastar_j}^2.
\end{align*}

\begin{corollary}[Linear Parameter Recovery Bound]
Suppose that the conditions of Theorem \ref{thm:convergence_linear_fresh} hold. Then, for every $\delta\in(0,1)$, with probability at least $1-\delta$,
\begin{align*}
\frac{1}{T}\sum_{t=0}^{T-1}\frac{1}{|\cH|}\sum_{j\in\cH}
\norm{\theta^{(t)}-\thetastar_j}^2&=\frac{1}{T}\sum_{t=0}^{T-1}\frac{1}{|\cH|}\sum_{j\in\cH}
\mathbb{E}_X\left[\norm{(\theta^{(t)}-\thetastar_j)X}^2\right]\\
&\lesssim\frac{\Delta L^{(0)}}{\eta T}
+\left(1+\kappa+\frac{1}{|\cH|}\right)\Gamma_{\textsf{lin}}
+R^2\sigma^2\left(\kappa+\frac{1}{|\cH|}\right)
\left(\frac{\vd\vdy+\log(2|\cH|T/\delta)}{\tau}\right).
\end{align*}
\end{corollary}

\begin{proof}
For every iteration $t$, adding and subtracting $\bar\theta^\star$ and applying Lemma \ref{lem:variance_decomposition} give
\begin{align*}
\frac{1}{|\cH|}\sum_{j\in\cH}
\norm{\theta^{(t)}-\thetastar_j}^2
&=\norm{\theta^{(t)}-\bar\theta^\star}^2
+\frac{1}{|\cH|}\sum_{j\in\cH}
\norm{\thetastar_j-\bar\theta^\star}^2.
\end{align*}
The population-gradient identity \eqref{eq:population_linear_gradient} implies
\begin{align*}
\norm{\theta^{(t)}-\bar\theta^\star}^2
=\frac{1}{4}\norm{\nabla L_\cH(\theta^{(t)})}^2.
\end{align*}
Moreover, applying the pairwise form of the variance decomposition to the client parameters yields
\begin{align*}
\frac{1}{|\cH|}\sum_{j\in\cH}
\norm{\thetastar_j-\bar\theta^\star}^2
=\frac{1}{2|\cH|^2}\sum_{i\in\cH}\sum_{j\in\cH}
\norm{\thetastar_i-\thetastar_j}^2
=\frac{1}{2}\Gamma_{\textsf{lin}}.
\end{align*}
Hence, combining the last three bounds and averaging over the iterations yields
\begin{align*}
\frac{1}{T}\sum_{t=0}^{T-1}\frac{1}{|\cH|}\sum_{j\in\cH}
\norm{\theta^{(t)}-\thetastar_j}^2
=\frac{1}{4}Q_T+\frac{1}{2}\Gamma_{\textsf{lin}}.
\end{align*}
Finally, substituting the bound on $Q_T$ from Theorem \ref{thm:convergence_linear_fresh} and absorbing numerical constants proves the result.
\end{proof}

\section{Nonlinear Regression}
\label{appendix:nonlinear}

In this section, we extend our analysis to the nonlinear regression setting and derive non-asymptotic guarantees for adversarially robust federated learning. Our analysis follows the same underlying structure in linear regression, while accounting for the additional intricacies introduced by the nonlinear parameterization. In particular, we first characterize the gradient decomposition across honest clients in function space and derive a bound on the resulting gradient heterogeneity under standard regularity and concentration conditions. We then combine this bound with the optimization-error recursion to establish a non-asymptotic ergodic convergence bound and the corresponding function-recovery guarantee.

\subsection{Regression Model and Gradients}
\label{subsec:nonlinear_model}
 
We recall that each honest client $i\in\cH$ has a ground-truth parameter $\thetastar_i\in\R^p$
and generates data according to
\begin{align*}
Y_{i,k} = h_{\thetastar_i}(X_{i,k}) + V_{i,k}, \forall k=1,\dots,\tau,
\end{align*}
where $h_\theta:\R^{\vd}\to\R^{\vdy}$ is a nonlinear function parameterized by $\theta\in\R^p$ (note that now $\theta$ is a vector that parameterizes the nonlinear function $h_\theta(\cdot)$), $V_{i,k}\in\R^{\vdy}$
satisfies Assumption \ref{assm:subg}, and $\|X_{i,k}\|\leq R$ almost surely. At a fixed iteration, with the iteration index on the fresh samples suppressed, we recall that the empirical loss is
\begin{align*}
\hat  L_i^{(t)}(\theta) := \frac{1}{\tau}\sum_{k=1}^{\tau}\norm{Y_{i,k} - h_\theta(X_{i,k})}^2.
\end{align*}
 
We then define the residual $r_{i,k}(\theta):=h_\theta(X_{i,k})-Y_{i,k}$, and write 
\begin{align*}
\nabla_\theta \hat  L_i^{(t)}(\theta) = \frac{2}{\tau}\sum_{k=1}^\tau r_{i,k}(\theta) J_\theta(X_{i,k})^\top.
\end{align*}
and by substituting the ground-truth model $Y_{i,k} = h_{\thetastar_i}(X_{i,k})+V_{i,k}$, we have 
\begin{align*}
\nabla_\theta \hat  L_i^{(t)}(\theta^{(t)})
= \frac{2}{\tau}\sum_{k=1}^\tau\bigl(h_{\theta^{(t)}}(X_{i,k})-h_{\thetastar_i}(X_{i,k})\bigr)
J_{\theta^{(t)}}(X_{i,k})^\top
- \frac{2}{\tau}\sum_{k=1}^\tau V_{i,k}\,J_{\theta^{(t)}}(X_{i,k})^\top.
\end{align*}

\subsection{Gradient Heterogeneity Decomposition}
\label{subsec:nl_grad_diff}

At iteration $t$, let
\begin{align*}
\hat  s_i^{(t)}&:=\frac{1}{\tau}\sum_{k=1}^{\tau}
\bigl(h_{\theta^{(t)}}(X_{i,k})-h_{\thetastar_i}(X_{i,k})\bigr)\Jt(X_{i,k})^\top,\\
s_i^{(t)}&:=\mathbb{E}_X\left[\bigl(h_{\theta^{(t)}}(X)-h_{\thetastar_i}(X)\bigr)\Jt(X)^\top\right].
\end{align*}
Here and below, the iteration index on the fresh samples is suppressed. We also define the client noise term and its honest average as
\begin{align*}
\xi_i^{(t)}:=\frac{2}{\tau}\sum_{k=1}^{\tau}V_{i,k}\Jt(X_{i,k})^\top \text{ and }
\bar\xi^{(t)}:=\frac{1}{|\cH|}\sum_{i\in\cH}\xi_i^{(t)}.
\end{align*}
Then, we have 
\begin{align}\label{eq:grad_diff_g}
\nabla \hat  L_i^{(t)}(\theta^{(t)})-\nabla \hat  L_\cH^{(t)}(\theta^{(t)})
=2\left(\hat  s_i^{(t)}-\frac{1}{|\cH|}\sum_{j\in\cH}\hat  s_j^{(t)}\right)
-\left(\xi_i^{(t)}-\bar\xi^{(t)}\right).
\end{align}

\subsection{Concentration Inequalities for Nonlinear Regression}
\label{subsec:nonlinear_concentration}

The fresh-sample assumption ensures that, conditionally on $\theta^{(t)}$, the samples used at iteration $t$ are independent of the current iterate.

\begin{lemma}\label{lem:emp_pop}
Suppose that Assumptions \ref{assm:jacobian} and \ref{assumption:subGaussian_residual} hold. Then, for every $\delta\in(0,1)$, with probability at least $1-\delta$, simultaneously for all $i\in\cH$ and $t=0,\ldots,T-1$, it holds that
\begin{align}\label{lem:e2}
\norm{\hat  s_i^{(t)}-s_i^{(t)}}_F
\leq C_1\bar J\sqrt{E_i^{(t)}}
\sqrt{\frac{p+\log(2|\cH|T/\delta)}{\tau}}.
\end{align}
\end{lemma}

\begin{proof}
We begin by noting that conditionally on $\theta^{(t)}$, the matrices 
\begin{align*}
W_{i,k}^{(t)}-\mathbb{E}W_{i,k}^{(t)} \text{ with }
W_{i,k}^{(t)}:=\bigl(h_{\theta^{(t)}}(X_{i,k})-h_{\thetastar_i}(X_{i,k})\bigr)\Jt(X_{i,k})^\top,
\end{align*}
are independent and sub-Gaussian from Assumption \ref{assumption:subGaussian_residual}. Then, Assumptions \ref{assm:jacobian} and \ref{assumption:subGaussian_residual} yield
\begin{align*}
\norm{W_{i,k}^{(t)}-\mathbb{E}W_{i,k}^{(t)}}_{\psi_2}
\leq C_1\bar J\sqrt{E_i^{(t)}}.\end{align*}

Therefore, the result follows from Lemma \ref{lem:subg_mean} and a union bound over the honest clients and iterations.
\end{proof}

The same argument, together with Assumption \ref{assm:subg}, yields
\begin{align}\label{eq:noise_bound_nl}
\norm{\xi_i^{(t)}}_F
&\leq C_1\bar J\sigma\sqrt{\frac{\textcolor{black}{p}+\log(2|\cH|T/\delta)}{\tau}},\notag\\
\norm{\bar\xi^{(t)}}_F
&\leq C_1\bar J\sigma\sqrt{\frac{\textcolor{black}{p}+\log(2T/\delta)}{\tau|\cH|}},
\end{align}
simultaneously for all honest clients and iterations, with probability at least $1-\delta$. 

\subsection{Per-iteration Gradient Heterogeneity Bound}
\label{subsec:nonlinear_heterogeneity}

We recall that 
\begin{align*}
\Gamma_{\textsf{nonlin}}
&:=\frac{1}{|\cH|^2}\sum_{i\in\cH}\sum_{j\in\cH}
\mathbb{E}_X\left[\norm{h_{\thetastar_i}(X)-h_{\thetastar_j}(X)}^2\right],\\
\bar E^{(t)}
&:=\frac{1}{|\cH|}\sum_{i\in\cH}
\mathbb{E}_X\left[\norm{h_{\theta^{(t)}}(X)-h_{\thetastar_i}(X)}^2\right],
\end{align*}
and, for some $\bar\omega>0$, let
\begin{align}\label{eq:omega_tau}
\bar\omega^2:=C_1\bar J^2\left(\frac{\textcolor{black}{p}+\log(2|\cH|T/\delta)}{\tau}\right).
\end{align}

\begin{lemma}\label{lem:Gt_nonlinear}
Suppose that Assumptions \ref{assm:jacobian}, \ref{assumption:subGaussian_residual}, and \ref{assm:subg} hold. Then, with probability at least $1-\delta$, simultaneously for all $t=0,\ldots,T-1$,
\begin{align*}
G^{(t)}\lesssim \bar J^2\Gamma_{\textsf{nonlin}}+\bar\omega^2\bar E^{(t)}
+\bar J^2\sigma^2\left(\frac{\textcolor{black}{p}+\log(2|\cH|T/\delta)}{\tau}\right).
\end{align*}
\end{lemma}

\begin{proof}
By adding and subtracting $s^{(t)}_i$ and $\frac{1}{|\cH|}\sum_{j\in\cH}s_j^{(t)}$ yields
\begin{align*}
\hat  s_i^{(t)}-\frac{1}{|\cH|}\sum_{j\in\cH}\hat  s_j^{(t)} &=s_i^{(t)}-\frac{1}{|\cH|}\sum_{j\in\cH}s_j^{(t)}+\hat  s_i^{(t)}-s_i^{(t)}-\frac{1}{|\cH|}\sum_{j\in\cH}\left(\hat  s_j^{(t)}-s_j^{(t)}\right).
\end{align*}
The first term contains only the model heterogeneity since the terms involving $h_{\theta^{(t)}}(X)$ cancel. By Jensen's inequality and Assumption \ref{assm:jacobian},
\begin{align*}
\frac{1}{|\cH|}\sum_{i\in\cH}
\norm{s_i^{(t)}-\frac{1}{|\cH|}\sum_{j\in\cH}s_j^{(t)}}_F^2
\leq \bar J^2\Gamma_{\textsf{nonlin}}.
\end{align*}
Lemma \ref{lem:emp_pop}, Jensen's inequality, \eqref{lem:e2}, and \eqref{eq:noise_bound_nl} bound the remaining terms. We then substitute these bounds into \eqref{eq:grad_diff_g}, squaring, and averaging over the honest clients proves the result.
\end{proof}

Let
\begin{align}\label{eq:Zt_def}
Z^{(t)}:=\norm{\nabla\hat  L_\cH^{(t)}(\theta^{(t)})-\nabla L_\cH(\theta^{(t)})}^2.
\end{align}
We now derive the bound on $Z^{(t)}$ explicitly. By the definitions of $\hat  s_i^{(t)}$, $s_i^{(t)}$, and $\bar\xi^{(t)}$, we have 
\begin{align*}
\nabla\hat  L_\cH^{(t)}(\theta^{(t)})-\nabla L_\cH(\theta^{(t)})
=\frac{2}{|\cH|}\sum_{i\in\cH}
\left(\hat  s_i^{(t)}-s_i^{(t)}\right)-\bar\xi^{(t)}.
\end{align*}
It follows from $\norm{a+b}^2\leq2\norm{a}^2+2\norm{b}^2$ that
\begin{align*}
Z^{(t)}
&\leq8\norm{\frac{1}{|\cH|}\sum_{i\in\cH}
\left(\hat  s_i^{(t)}-s_i^{(t)}\right)}^2
+2\norm{\bar\xi^{(t)}}^2.
\end{align*}
We note that conditionally on $\theta^{(t)}$, the matrices $\hat  s_i^{(t)}-s_i^{(t)}$ are independent and mean-zero across honest clients. \textcolor{black}{In addition, for every $U\in\mathbb S_F^{p\times1}$, conditional sub-Gaussian Hoeffding \citep[Theorem 2.7.3]{vershynin2018high} implies
\begin{align*}
&\mathbb P\left(\left|\left\langle U,\frac{1}{|\cH|}\sum_{i\in\cH}(\hat s_i^{(t)}-s_i^{(t)})\right\rangle_F\right|\geq u\middle|\theta^{(t)}\right) \leq2\exp\left(-\frac{c\tau|\cH|^2u^2}{\bar J^2\sum_{i\in\cH}E_i^{(t)}}\right).
\end{align*}
Therefore, by taking a union bound over a fixed net of $\mathbb S_F^{p\times1}$ and over the iterations yields the following expression}
\begin{align*}
\norm{\frac{1}{|\cH|}\sum_{i\in\cH}
\left(\hat  s_i^{(t)}-s_i^{(t)}\right)}^2
&\lesssim\frac{\bar\omega^2}{|\cH|}
\left(\frac{1}{|\cH|}\sum_{i\in\cH}E_i^{(t)}\right) =\frac{\bar\omega^2}{|\cH|}\bar E^{(t)}.
\end{align*}
Moreover, the second inequality in \eqref{eq:noise_bound_nl} provides the bound for the second term in the bound of $Z^{(t)}$, i.e.,
\begin{align*}
\norm{\bar\xi^{(t)}}^2
\lesssim\bar J^2\sigma^2
\left(\frac{\textcolor{black}{p}+\log(2T/\delta)}{\tau|\cH|}\right).
\end{align*}
Therefore, combining these bounds proves
\begin{align}\label{eq:Zt_bound}
Z^{(t)}\lesssim\frac{\bar\omega^2}{|\cH|}\bar E^{(t)}
+\bar J^2\sigma^2\left(\frac{\textcolor{black}{p}+\log(2T/\delta)}{\tau|\cH|}\right),
\end{align}
simultaneously over the iterations, with probability at least $1-\delta$.

\subsection{Bound on \texorpdfstring{$G_T$}{GT} and Convergence Analysis}\label{subsec:nonlinear_convergence}

Let $\thetastar$ be a minimizer of the honest-average population loss $L_\mathcal{H}(\theta)$. We note that conditional mean-zero noise implies
\begin{align}\label{eq:E_loss_relation}
\bar E^{(t)}=L_\cH(\theta^{(t)})-L_\cH(\thetastar)+\bar E^\star \text{ where we have }
\bar E^\star:=\frac{1}{|\cH|}\sum_{i\in\cH}
\mathbb{E}_X\left[\norm{h_{\thetastar}(X)-h_{\thetastar_i}(X)}^2\right].
\end{align}
As $\thetastar$ minimizes $L_\cH(\theta)$, by evaluating the loss at each client optimum and averaging provides $\bar E^\star\leq\Gamma_{\textsf{nonlin}}$. Assumption \ref{ass:PL} therefore yields
\begin{align}\label{eq:E_PL_relation}
\bar E^{(t)}\leq\Gamma_{\textsf{nonlin}}
+\frac{1}{2\mu^\prime}\norm{\nabla L_\cH(\theta^{(t)})}^2.
\end{align}

\begin{theorem}\label{thm:convergence_nl}
Suppose that Assumptions \ref{assm:subg},  \ref{assm:jacobian}, \ref{assumption:subGaussian_residual}, \ref{ass:smooth}, and \ref{ass:PL} hold. Let $\eta\leq1/L^\prime$ and suppose that the number of fresh samples per client satisfies
\begin{align}\label{eq:sample_cond_nl}
\tau\geq\frac{4\textcolor{black}{C_0}C_1\bar J^2}{\mu^\prime}
\left(\kappa+\frac{1}{|\cH|}\right)
\left(\textcolor{black}{p}+\log\left(\frac{2|\cH|T}{\delta}\right)\right).
\end{align}
Then, for every $\delta\in(0,1)$, with probability at least $1-\delta$,
\begin{align*}
Q_T&:=\frac{1}{T}\sum_{t=0}^{T-1}\norm{\nabla L_\cH(\theta^{(t)})}^2\lesssim\frac{\Delta L^{(0)}}{\eta T}+\textcolor{black}{\left[\kappa\bar J^2+\bar\omega^2\left(\kappa+\frac{1}{|\cH|}\right)\right]\Gamma_{\textsf{nonlin}}}\\
&+\bar J^2\sigma^2\left(\kappa+\frac{1}{|\cH|}\right)
\left(\frac{\textcolor{black}{p}+\log(2|\cH|T/\delta)}{\tau}\right),
\end{align*}
where $\Delta L^{(0)}:=L_\cH(\theta^{(0)})-L_\cH(\thetastar)$.
\end{theorem}

\begin{proof}
Let $e^{(t)}:=\sF^{(t)}-\nabla L_\cH(\theta^{(t)})$. We note that the robustness of the aggregation rule and \eqref{eq:Zt_def} imply
\begin{align}\label{eq:agg_error_nl}
\norm{e^{(t)}}^2\leq2\kappa G^{(t)}+2Z^{(t)}.
\end{align}
By $L^\prime$-smoothness and $\eta\leq1/L^\prime$,
\begin{align*}
L_\cH(\theta^{(t+1)})-L_\cH(\theta^{(t)})
\leq-\frac{\eta}{2}\norm{\nabla L_\cH(\theta^{(t)})}^2
+\frac{\eta}{2}\norm{e^{(t)}}^2.
\end{align*}
Then, we sum the above inequality over the iterations and use the optimality of $\thetastar$ to obtain
\begin{align*}
Q_T &\leq\frac{2\Delta L^{(0)}}{\eta T}
+\frac{1}{T}\sum_{t=0}^{T-1}\norm{e^{(t)}}^2 \leq\frac{2\Delta L^{(0)}}{\eta T}+2\kappa G_T+2Z_T,
\end{align*}
where we have $Z_T:=\frac{1}{T}\sum_{t=0}^{T-1}Z^{(t)}$. By averaging Lemma \ref{lem:Gt_nonlinear} and \eqref{eq:Zt_bound}, we obtain
\begin{align*}
G_T
&\lesssim\bar J^2\Gamma_{\textsf{nonlin}}+\bar\omega^2\bar E_T
+\bar J^2\sigma^2\left(\frac{\textcolor{black}{p}+\log(2|\cH|T/\delta)}{\tau}\right),\\
Z_T
&\lesssim\frac{\bar\omega^2}{|\cH|}\bar E_T
+\bar J^2\sigma^2\left(\frac{\textcolor{black}{p}+\log(2|\cH|T/\delta)}{\tau|\cH|}\right),
\end{align*}
where $\bar E_T:=\frac{1}{T}\sum_{t=0}^{T-1}\bar E^{(t)}$. In addition, we \eqref{eq:E_PL_relation} over iterations to obtain
\begin{align*}
\bar E_T\leq\Gamma_{\textsf{nonlin}}+\frac{Q_T}{2\mu^\prime}.
\end{align*}
Then, by combining the preceding bounds, we have 
\begin{align*}
Q_T &\lesssim\frac{\Delta L^{(0)}}{\eta T}
+\textcolor{black}{\left[\kappa\bar J^2+\bar\omega^2\left(\kappa+\frac{1}{|\cH|}\right)\right]\Gamma_{\textsf{nonlin}}}\\
&+\bar J^2\sigma^2\left(\kappa+\frac{1}{|\cH|}\right)
\left(\frac{\textcolor{black}{p}+\log(2|\cH|T/\delta)}{\tau}\right)+\frac{\bar\omega^2}{\mu^\prime}
\left(\kappa+\frac{1}{|\cH|}\right)Q_T.
\end{align*}
Finally, substituting \eqref{eq:omega_tau} into \eqref{eq:sample_cond_nl} yields
\begin{align*}
\frac{\bar\omega^2}{\mu^\prime}
\left(\kappa+\frac{1}{|\cH|}\right)\leq\frac{1}{\textcolor{black}{C_0}},
\end{align*}
\textcolor{black}{As $C_0$ is chosen sufficiently large relative to the universal constants hidden above,} this allows us to move the final term to the left-hand side and absorb the resulting numerical factor into the constant absorbed in $\lesssim$.
\end{proof}

We now combine Lemma \ref{lem:Gt_nonlinear}, \eqref{eq:E_PL_relation}, and Theorem \ref{thm:convergence_nl} to write the following bound.

\begin{lemma}\label{lem:GT_nonlinear}
Suppose that the conditions of Theorem \ref{thm:convergence_nl} hold. Then, with probability at least $1-\delta$,
\begin{align*}
G_T\lesssim\textcolor{black}{(\bar J^2+\bar\omega^2)\Gamma_{\textsf{nonlin}}}
+\frac{\bar\omega^2\Delta L^{(0)}}{\mu^\prime\eta T}
+\bar J^2\sigma^2\left(\frac{\textcolor{black}{p}+\log(2|\cH|T/\delta)}{\tau}\right).
\end{align*}
\end{lemma}

\begin{proof}
We begin by averaging the per-iteration bound in Lemma \ref{lem:Gt_nonlinear} gives
\begin{align*}
G_T\lesssim\bar J^2\Gamma_{\textsf{nonlin}}+\bar\omega^2\bar E_T
+\bar J^2\sigma^2\left(\frac{\textcolor{black}{p}+\log(2|\cH|T/\delta)}{\tau}\right).
\end{align*}
By the averaged form of \eqref{eq:E_PL_relation}, we then have
\begin{align*}
\bar E_T\leq\Gamma_{\textsf{nonlin}}+\frac{Q_T}{2\mu^\prime}.
\end{align*}
Therefore, by substituting this inequality and then applying Theorem \ref{thm:convergence_nl}, we obtain
\begin{align*}
G_T
&\lesssim(\bar J^2+\bar\omega^2)\Gamma_{\textsf{nonlin}}
+\frac{\bar\omega^2}{\mu^\prime}
\left(\frac{\Delta L^{(0)}}{\eta T}
+\textcolor{black}{\left[\kappa\bar J^2+\bar\omega^2\left(\kappa+\frac{1}{|\cH|}\right)\right]\Gamma_{\textsf{nonlin}}}\right)\\
&+\bar J^2\sigma^2\left(1+
\frac{\bar\omega^2}{\mu^\prime}
\left(\kappa+\frac{1}{|\cH|}\right)\right)
\left(\frac{\textcolor{black}{p}+\log(2|\cH|T/\delta)}{\tau}\right).
\end{align*}
The sample-size condition \eqref{eq:sample_cond_nl} bounds the factors involving $\bar\omega^2$, and absorbing universal numerical constants yields the stated result.
\end{proof}

\subsection{Non-Asymptotic Function Recovery Error Bound}
\label{subsec:param_recovery_nonlinear}

We conclude the nonlinear analysis by restating the function-recovery guarantee and providing its complete derivation.

\begin{corollary}[Parameter Recovery Bound]
Suppose that the conditions of Theorem \ref{thm:convergence_nl} hold. Then, for every $\delta\in(0,1)$, with probability at least $1-\delta$, it holds that
\begin{align*}
&\frac{1}{T}\sum_{t=0}^{T-1}\frac{1}{|\cH|}\sum_{j\in\cH}
\mathbb{E}_X\left[\norm{h_{\theta^{(t)}}(X)-h_{\thetastar_j}(X)}^2\right]\\
&\lesssim\frac{\Delta L^{(0)}}{\mu^\prime\eta T}
+\textcolor{black}{\left[1+\frac{\kappa\bar J^2}{\mu^\prime}
+\frac{\bar\omega^2}{\mu^\prime}\left(\kappa+\frac{1}{|\cH|}\right)\right]\Gamma_{\textsf{nonlin}}}+\frac{\bar J^2\sigma^2}{\mu^\prime}
\left(\kappa+\frac{1}{|\cH|}\right)
\left(\frac{\textcolor{black}{p}+\log(2|\cH|T/\delta)}{\tau}\right).
\end{align*}
\end{corollary}

\begin{proof}
Recall from \eqref{eq:E_loss_relation} that
\begin{align*}
\bar E^{(t)}
=L_\cH(\theta^{(t)})-L_\cH(\thetastar)+\bar E^\star.
\end{align*}
To control $\bar E^\star$, evaluate the average functional error at each client optimum and then average over these choices. As $\thetastar$ minimizes the honest-average population loss, we have
\begin{align*}
\bar E^\star
&\leq\frac{1}{|\cH|}\sum_{i\in\cH}
\left(\frac{1}{|\cH|}\sum_{j\in\cH}
\mathbb{E}_X\left[\norm{h_{\thetastar_i}(X)-h_{\thetastar_j}(X)}^2\right]\right)\\
&=\frac{1}{|\cH|^2}\sum_{i\in\cH}\sum_{j\in\cH}
\mathbb{E}_X\left[\norm{h_{\thetastar_i}(X)-h_{\thetastar_j}(X)}^2\right]
:=\Gamma_{\textsf{nonlin}}.
\end{align*}
We note that the PL condition (Assumption \ref{ass:PL}) implies
\begin{align*}
L_\cH(\theta^{(t)})-L_\cH(\thetastar)
\leq\frac{1}{2\mu^\prime}
\norm{\nabla L_\cH(\theta^{(t)})}^2.
\end{align*}
Therefore, by combining the previous bounds, we obtain
\begin{align*}
\bar E^{(t)}
\leq\Gamma_{\textsf{nonlin}}+\frac{1}{2\mu^\prime}
\norm{\nabla L_\cH(\theta^{(t)})}^2.
\end{align*}
Finally, averaging over $t=0,\ldots,T-1$ yields
\begin{align*}
\frac{1}{T}\sum_{t=0}^{T-1}\bar E^{(t)}
\leq\Gamma_{\textsf{nonlin}}+\frac{Q_T}{2\mu^\prime}.
\end{align*}
Hence, substituting the bound on $Q_T$ from Theorem \ref{thm:convergence_nl} and absorbing numerical constants proves the result.
\end{proof}

\end{document}